\documentclass[twoside,11pt]{article}

\usepackage{blindtext}

\usepackage[nohyperref, preprint]{jmlr2e}

\usepackage{algorithmic,algorithm}
\usepackage{microtype}
\usepackage{nicefrac}       

\usepackage[colorlinks,citecolor=blue,urlcolor=blue,linkcolor=blue,linktocpage=true]{hyperref}

\usepackage{multirow,makecell}
\usepackage{mathtools}
\usepackage{bbm}
\usepackage{enumerate}
\usepackage{subfigure}
\usepackage{color}

\newtheorem{assumption}[theorem]{Assumption}
\newcommand{\argmin}{\mathop{\mathrm{argmin}}}
\newcommand{\argmax}{\mathop{\mathrm{argmax}}}

\def\R{\mathbb{R}}
\def\E{\mathbb{E}}

\def\1{\mathbbm{1}}

\def\half{\frac{1}{2}}

\def\cA{\mathcal{A}}

\def\cF{\mathcal{F}}

\def\cH{\mathcal{H}}

\def\cW{\mathcal{W}}
\def\cX{\mathcal{X}}
\def\cY{\mathcal{Y}}

\usepackage{lastpage}
\jmlrheading{23}{2024}{1-\pageref{LastPage}}{1/21; Revised 5/22}{9/22}{21-0000}{Liu, Qiao, Yin, Bogunovic, and Wang}

\ShortHeadings{No-Regret Linear Bandits under Gap-Adjusted Misspecification}{Liu, Qiao, Yin, Bogunovic, and Wang}
\firstpageno{1}

\begin{document}

\title{No-Regret Linear Bandits under Gap-Adjusted Misspecification}

\author{\name Chong Liu \email cliu24@albany.edu \\
\addr Department of Computer Science\\
       University at Albany, State University of New York\\
       Albany, NY 12222, USA
\AND
Dan Qiao \email d2qiao@ucsd.edu \\
\addr Department of Computer Science and Engineering\\
       University of California, San Diego\\
       La Jolla, CA 92093, USA
\AND
Ming Yin \email ming\_yin@ucsb.edu \\
\addr Department of Computer Science\\
University of California, Santa Barbara\\
Santa Barbara, CA 93106, USA
       \AND
       \name Ilija Bogunovic \email i.bogunovic@ucl.ac.uk \\
       \addr Department of Electronic and Electrical Engineering\\
University College London\\
London, WC1E 7JE, UK
       \AND
       \name Yu-Xiang Wang \email yuxiangw@ucsd.edu \\
       \addr Halıcıoğlu Data Science Institute\\
       University of California, San Diego\\
       La Jolla, CA 92093, USA}

\editor{My editor}

\maketitle

\begin{abstract}
This work studies linear bandits under a new notion of \emph{gap-adjusted} misspecification and is an extension of \citet{liu2023no}. When the underlying reward function is \emph{not} linear, existing linear bandits work usually relies on a uniform misspecification parameter $\epsilon$ that measures the sup-norm error of the best linear approximation. This results in an unavoidable linear regret whenever $\epsilon > 0$. We propose a more natural model of misspecification which only requires the approximation error at each input $x$ to be proportional to the suboptimality gap at $x$.  It captures the intuition that, for optimization problems, near-optimal regions should matter more and we can tolerate larger approximation errors in suboptimal regions.

Quite surprisingly, we show that the classical LinUCB algorithm --- designed for the realizable case --- is automatically robust against such $\rho$-gap-adjusted misspecification with parameter $\rho$ diminishing at $O(1/(d \sqrt{\log T}))$. It achieves a near-optimal $O(\sqrt{T})$ regret for problems that the best-known regret is almost linear in time horizon $T$. We further advance this frontier by presenting a novel phased elimination-based algorithm whose \emph{gap-adjusted} misspecification parameter $\rho = O(1/\sqrt{d})$ does not scale with $T$. This algorithm attains optimal $O(\sqrt{T})$ regret and is deployment-efficient, requiring only $\log T$ batches of exploration. It also enjoys an adaptive $O(\log T)$ regret when a constant suboptimality gap exists.  Technically, our proof relies on a novel self-bounding argument that bounds the part of the regret due to misspecification by the regret itself, and a new inductive lemma that limits the misspecification error within the suboptimality gap for all valid actions in each batch selected by G-optimal design.
\end{abstract}


\section{Introduction}\label{sec:intro}
Stochastic linear bandit is a classical problem of online learning and decision-making with many influential applications, e.g., A/B testing \citep{claeys2021dynamic}, recommendation systems \citep{chu2011contextual}, advertisement placements \citep{wang2021hybrid}, clinical trials \citep{moradipari2020stage}, hyperparameter tuning \citep{alieva2021robust}, and new material discovery \citep{katz2020empirical}.

More formally, the classical stochastic bandit problem which is a game between a player and nature. The goal of the player is to find the optimal point $x_*$ that is able to maximize the unknown objective function $f_0$:
\begin{align*}
x_* = \argmax_{x \in \cX} f_0(x),
\end{align*}
where $\cX \subseteq \mathbb{R}^d$ is the function domain. The game runs in iterations. For all iteration $t \in [T]$, the player chooses an action $x_t \in \cX$ and nature will release the noisy feedback:
\begin{align*}
y_t = f_0(x_t) + \eta_t,
\end{align*}
where the noise $\eta_t$ is independent, zero-mean, and $\sigma$-sub-Gaussian noise. Equivalently, the goal of the player is to minimize the cumulative regret the player has relative to an oracle who knows $x_*$ ahead of time, i.e.,
\begin{align}
R_T = \sum_{t=1}^T r_t = \sum_{t=1}^T f_0(x_*) - f_0(x_t),\label{eq:cr}
\end{align}
where $r_t$ is called instantaneous regret.

Despite being highly successful in the wild, existing theory for stochastic linear bandits \citep{abbasi2011improved,lattimore2020learning}, or more generally learning-oracle based bandits problems, relies on a \emph{realizability} assumption, i.e., the learner is given access to a function class $\cF$ such that the true expected reward $f_0: \cX\rightarrow \R$ satisfies that $f_0\in\cF$. Although realizability paves the way for a lot of theoretical work in both bandits \citep{srinivas2010gaussian} and reinforcement learning \citep{zhan2022offline}, it is considered one of the strongest and most restrictive assumptions in the standard statistical learning setting, but in the bandits theory literature the common sentiment is that it is mild and acceptable, despite the ``elephant in the room'' that everyone sees but voluntarily ignores — Realizability is never true in practice! For example, in recommendation systems problem \citep{chu2011contextual}, the reward function $f_0$ cannot be a perfect linear function w.r.t. the action feature vector. However, why is this the case? The argument to justify the realizability assumption is legitmate: all known attempts to deviate from the realizability assumption results in a regret that grows linearly with $T$ \citep{lattimore2020learning,bogunovic2021misspecified}.

\begin{figure}[t]
	\centering     
	\subfigure[$0.7$-uniform misspecification]{\label{fig:1d_unif}\includegraphics[width=0.49\linewidth]{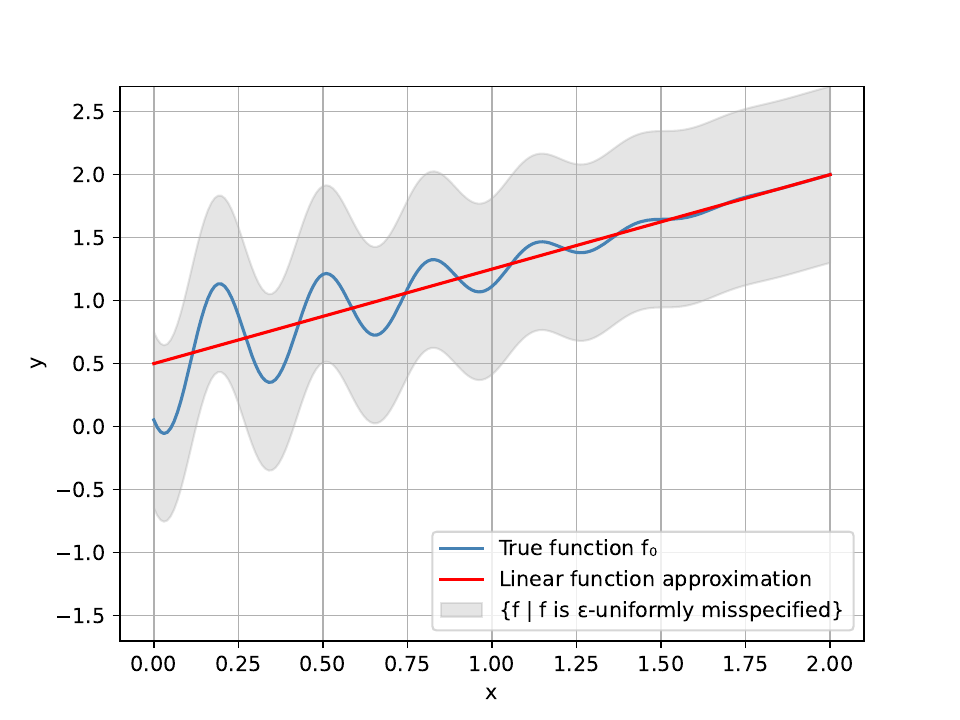}}
	\subfigure[$0.7$-gap-adjusted misspecification]{\label{fig:1d_gam}\includegraphics[width=0.49\linewidth]{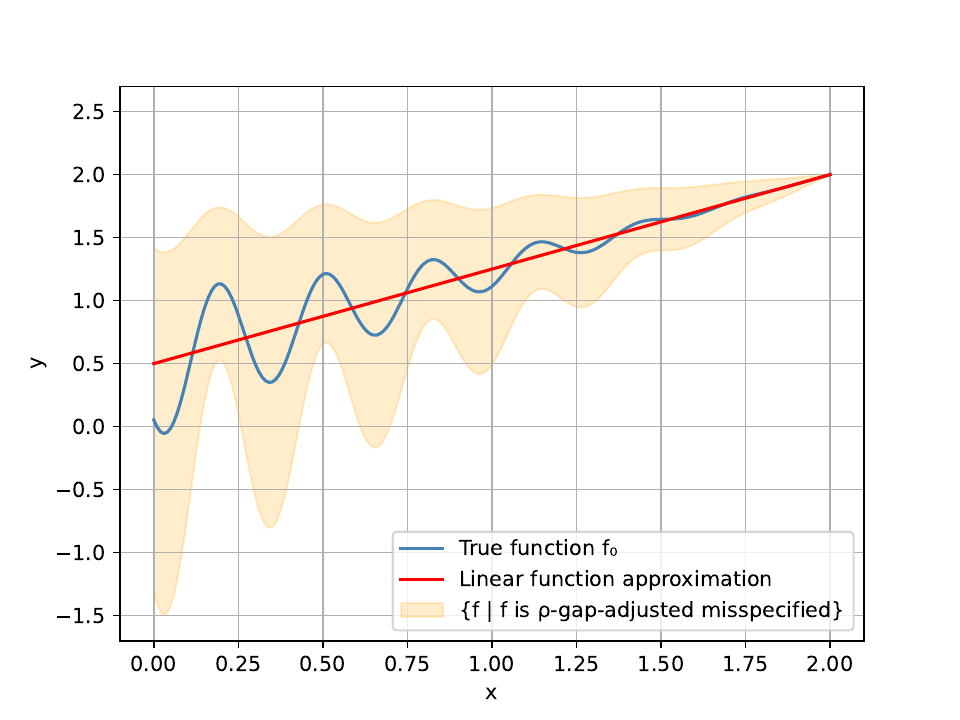}}
	\caption{Examples of misspecification in 1 dimension. The blue line denotes the non-linear true function $f_0$ and the red line shows a feasible linear function that is able to optimize $f_0$ by taking $x_*=2$. (a) An example of $\epsilon$-uniform misspecification (Definition \ref{def:unif_mis}) where $\epsilon=0.7$. The gray region shows the uniformly misspecified function class. Note the vertical range of it is always $2\epsilon=1.4$ over the whole domain. (b) An example of $\rho$-gap-adjusted misspecification (Definition \ref{def:gap_adj_mis}) where $\rho=0.7$. The orange region shows the gap-adjusted misspecified function class. Note the vertical range at a certain point $x$ depends on the suboptimal gap. For example, the vertical range at $x=0$ is much larger than it at $x=1$ and there is no vertical range at $x_*=2$.}
	\label{fig:1d}
\end{figure}

\begin{figure}[t]
	\centering     
	\subfigure[$100$-uniform misspecification]{\label{fig:2d_unif}\includegraphics[width=0.49\linewidth]{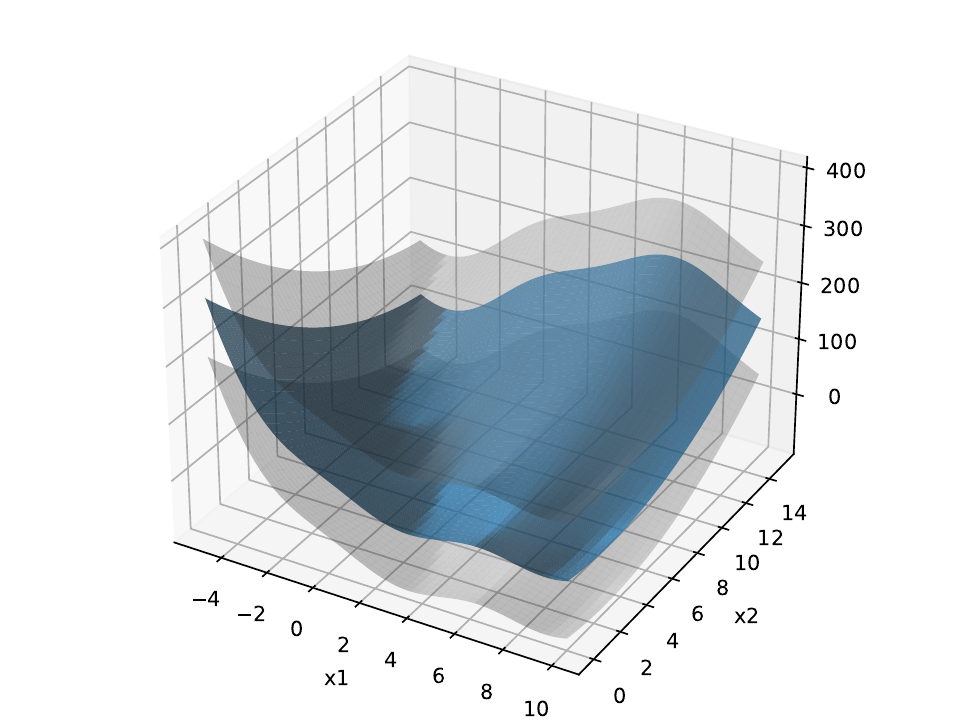}}
	\subfigure[$0.3$-gap-adjusted misspecification]{\label{fig:2d_gam}\includegraphics[width=0.49\linewidth]{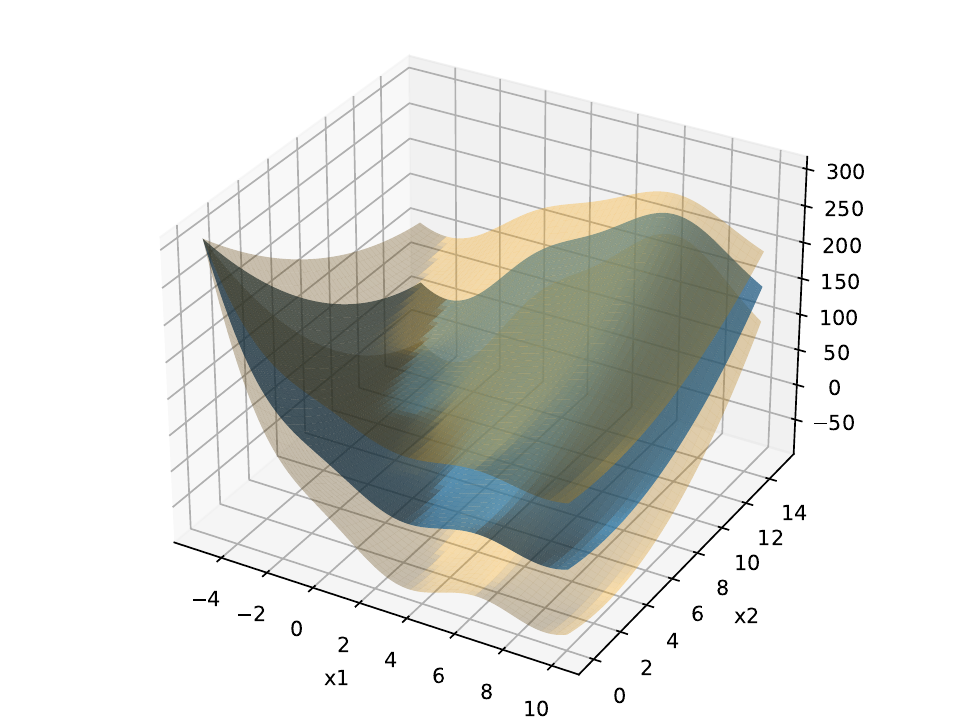}}
	\caption{Examples of misspecification in 2 dimensions. The blue surface denotes the Branin function $f_0$. The optimal point $x_*$ is $(x_1=-5,x_2=0)$. (a) An example of $\epsilon$-uniform misspecification (Definition \ref{def:unif_mis}) where $\epsilon=100$. Two gray surfaces denote the upper and lower bound of misspecified function class. (b) An example of $\rho$-gap-adjusted misspecification (Definition \ref{def:gap_adj_mis}) where $\rho=0.7$. Two orange surfaces denote the upper and lower bound of misspecified function class. Note there is no misspecification at $x_*$.}
	\label{fig:2d}
\end{figure}

To get rid of the realizability assumption, various works have studied the misspecified bandit problem and the most studied setting is $\epsilon$-uniform misspecifiation where the $\ell_\infty$ distance between the true function and the best-in-class approximation function is always upper bounded by the parameter $\epsilon$, defined as:

\begin{definition}[$\epsilon$-Uniform Misspecification]\label{def:unif_mis}
    Function class $\cF$ is an $\epsilon$-uniform misspecified approximation of $f_0$ if there exists $f\in \cF$ such that $\sup_{x\in\cX}|f(x) - f_0(x)| \leq \epsilon$.
\end{definition}

Under this definition, \citet{lattimore2020learning} studied linear bandits and \citet{bogunovic2021misspecified} studied Gaussian process bandits, however, both of them ended up in the $\tilde{O}(d\sqrt{T}+\epsilon \sqrt{d}T)$ cumulative regret regime, which (depending on $\epsilon$) can result in linear regret. The reason behind it is that $\epsilon$-uniform misspecification does not take the function structure into consideration, shown in Figure \ref{fig:1d_unif} and Figure \ref{fig:2d_unif}. In both cases, at each round, whatever the player learns about the true objective function, the best thing that the player can hope for to upper bound the misspecifiation error is only $\epsilon$. That is why the lower bound in \citet{bogunovic2021misspecified} shows that the $\tilde{\Omega}(\epsilon T)$ term is unavoidable in this setting.

In practical applications, it is often observed that feature-based representation of the actions with function approximations in
estimating the reward can result in very strong policies even if the estimated reward functions are far from
being correct \citep{foster2018practical}.

So what went wrong? The critical observation that we rely on is the following:
\begin{quote}
    Intuitively, it should be sufficient for the estimated reward function to clearly \emph{differentiate} good actions from bad ones, rather than requiring the function to perfectly estimate the rewards numerically. 
\end{quote}

Based on the issue raised by uniform misspecification and the key observation above, we define a new notation of model misspecification, which is called $\rho$-gap-adjusted misspecification. Under this new defintion, the $\ell_\infty$ distance between the true function and the best-in-class approximation function normalized by its suboptimality gap is always upper bounded by the parameter $\rho$, i.e.,

\begin{definition}[$\rho$-Gap-Adjusted Misspecification ($\rho$-GAM)]\label{def:gap_adj_mis}
Denote $f^*=\max_{x\in\mathcal{X}} f(x)$. Then $f$ is $\rho$-gap-adjusted misspecification approximation of $f_0$ for a parameter $0 \leq \rho < 1$ if:
\begin{align*}
\sup_{x \in \cX} \left| \frac{f(x) -f_0(x)}{f^* - f_0(x)}\right|\leq \rho.
\end{align*}
\end{definition}

Note $\rho$ is the ratio between two distances, thus $\rho$ cannot be directly compared with $\epsilon$ in their values. Figure \ref{fig:1d_gam} and Figure \ref{fig:2d_gam} show two examples under this condition respectively. The key intuition behind the $\rho$-GAM condition is that near-optimal region should matter more in optimization and larger misspecification is allowed in suboptimal regions.

In this paper, we systematically and theoretically investigate linear bandits under this new $\rho$-gap-adjusted misspecfication condition and our \textbf{contributions} are summarized as follows.

\textbf{Contributions.} 
In this paper, we formalize this intuition by defining a new family of misspecified bandits problems based on a condition that adjusts the need for an accurate approximation pointwise at every $x\in\cX$ according to the suboptimality gap at $x$. Unlike the existing misspecified linear bandits problems with a linear regret, our problem admits a nearly optimal $\tilde{O}(\sqrt{T})$ regret despite being heavily misspecified. Specifically:
\begin{itemize}
 \item We define $\rho$-\emph{gap-adjusted misspecified} ($\rho$-GAM) function approximations and characterize how they preserve important properties of the true function that are relevant for optimization. To the best of our knowledge, the suboptimality-gap-adjusted misspecification problem was not studied before and we are the first to obtain $\sqrt{T}$-style regrets without a realizability assumption.
 \item We show that the classical LinUCB algorithm \citep{abbasi2011improved} can be used \emph{as is} (up to some mild hyperparameter) to achieve an $\tilde{O}(\sqrt{T})$ regret under a moderate level of gap-adjusted misspecification ($\rho \leq O(1/\sqrt{\log T})$). In comparison, the regret bound one can obtain under the corresponding uniform-misspecification setting is only $\tilde{O}(T/\sqrt{\log T})$. This represents an exponential improvement in the average regret metric $R_T/T$.
 \item By working with a modified Phased Elimination (PE) algorithm \citep{lattimore2020learning}, we prove that under the $\rho$-GAM condition, PE algorithm achieves the $\tilde{O}(\sqrt{T})$ regret where $\rho$ can be a constant. It will inspire more research under the $\rho$-GAM condition, including kernelized bandits and reinforcement learning. Moreover, we prove that if there exists a positive suboptimal gap $\Delta$, PE algorithm achieves the $\tilde{O}(\log T/\Delta)$ regret where $\rho$ is still a constant.
    \item As a by-product of the algorithm design, PE algorithm enjoys low $O(\log T)$ policy switching cost, which is highly deployment-efficient in practice. In constrast, policy switching cost of LinUCB algorithm \citep{abbasi2011improved} is $O(T)$.
\end{itemize}

\textbf{Technical novelty.} Due to misspecification, we have technical challenges that appear in bounding the instantaneous regret and the uncertainty region. Specifically,

\begin{enumerate}
    \item We tackle the challenge by self bounding trick, i.e., bounding the instantaneous regret by the instantaneous regret itself, which can be of independent interest in more settings, e.g., Gaussian process bandit optimization and reinforcement learning. 
    \item Due to the potential function lemma in LinUCB algorithm, $\rho \leq \tilde{O}(1/\sqrt{\log T})$ cannot be improved to constant order. To address this challenge, we work with the PE algorithm which maintains a tighter action version space rather than the parameter space in \citet{abbasi2011improved}. 
    \item Since misspecification under the $\rho$-GAM condition depends on suboptimality gap, it is different at different $x \in \cX$ and analysis in \citet{lattimore2020learning} doesn't work through. We prove a new inductive lemma that limits the misspecification error within the suboptimality gap for all valid actions in each batch selected by G-optimal design.
\end{enumerate}

The rest of the paper is organized as follows. In Section \ref{sec:rw}, we briefly review related work. In Section \ref{sec:pre}, we prepare all necessary materials before showing our main results, including notations, problem statement, all needed assumptions, LinUCB algorithm \citep{abbasi2011improved}, and a slightly modified phased elimination algorithm from \citet{lattimore2020learning}. We present theoretical results of LinUCB algorithm under gap-adjusted misspecification in Section \ref{sec:linucb} where $\rho=O(1/(d\sqrt{\log T}))$ is required. Section \ref{sec:pe} shows theoretical results of the new PE algorithm where $\rho=O(1/\sqrt{d})$ does not scale with $T$. Finally Section \ref{sec:conclusion} concludes the paper. For completeness, we show all technical lemmas in Appendix \ref{app:tech}, discuss the weaker $\rho$-gap-adjusted misspecification condition in Appendix \ref{app:weak}, and propose a unified misspecified linear bandit framework in Appendix \ref{app:unified}.

\section{Related Work}\label{sec:rw}

The problem of linear bandits was first introduced in \citet{abe1999associative}. Then \citet{auer2002finite} proposed the upper confidence bound to study linear bandits where the number of actions is finite. Based on it, \citet{dani2008stochastic} proposed an algorithm based on confidence ellipsoids and then \citet{abbasi2011improved} simplified the proof with a novel self-normalized martingale bound. Later \citet{chu2011contextual} proposed a simpler and more robust linear bandit algorithm and showed $\tilde{O}(\sqrt{d T})$ regret cannot be improved beyond a polylog factor. \citet{li2019nearly} further improved the regret upper and lower bound, which characterized the minimax regret up to an iterated logarithmic factor. See \citet{lattimore2020bandit} for a detailed survey of linear bandits. 

A standard assumption in the theory of stochastic linear bandits requires a well-specified reward function, i.e., a ``realizability'' assumption. This assumes that the expected reward linear in the action features, which might not be true in practice. To overcome this limitation, a recent line of research focuses on demonstrating the robustness of linear bandits algorithms by showing that if a linear function can approximate the underlying reward function up to an additive constant of $\epsilon$. Regret bound of the form $\tilde{O}(d\sqrt{T} + \epsilon \sqrt{d}T)$ has been shown to be minimax optimal under various settings \citep{foster2020beyond,neu2020efficient,zanette2020learning,lattimore2020learning,bogunovic2021misspecified,krishnamurthy2021tractable,foster2020adapting}

In terms of misspecification, \citet{ghosh2017misspecified} first studied the misspecified linear bandit with a fixed action set. They found that the LinUCB algorithm \citep{abbasi2011improved} is not robust when misspecification is large. They showed that in a favourable case when one can test the linearity of the reward function, their RLB algorithm is able to switch between the linear bandit algorithm and finite-armed bandit algorithm to address misspecification issue and achieve the $\tilde{O}(\min \{\sqrt{K},d\}\sqrt{T})$ regret where $K$ is number of arms.

The most studied setting of model misspecification is uniform misspecification where the $\ell_\infty$ distance between the best-in-class function and the true function is always upper bounded by some parameter $\epsilon$. Under this definition, \citet{lattimore2020learning} proposed the optimal design-based phased elimination algorithm for misspecified linear bandits and achieved $\tilde{O}(d\sqrt{T} + \epsilon \sqrt{d} T)$ regret when number of actions is infinite. They also found that with modified confidence band in LinUCB, LinUCB is able to achieve the same regret. With the same misspecification model, \citet{foster2020beyond} studied contextual bandit with regression oracle, \citet{neu2020efficient} studied multi-armed linear contextual bandit, and \citet{zanette2020learning} studied misspecified contextual linear bandits after reduction of the algorithm. All of their results suffer from linear regrets. Later \citet{bogunovic2021misspecified} studied misspecified Gaussian process bandit optimization problem and achieved $\tilde{O}(d\sqrt{T} + \epsilon \sqrt{d} T )$ regret when linear kernel is used in Gaussian process. Moreover, their lower bound shows that $\tilde{\Omega}(\epsilon T)$ term is unavoidable in this setting. 

Besides uniform misspecification, there are some work considering different definitions of misspecification in contextual bandits. \citet{krishnamurthy2021tractable} defines misspecification error as an expected squared error between true function and best-in-class function where expectation is taken over distribution of context space and action space. \citet{foster2020adapting} considered average misspecification, which is weaker than uniform misspecification and allows tighter regret bound. However, they also have linear regrets and their results do not directly apply to our problem because our action space is unbounded. 
Our work is different from all related work mentioned above because we are working under a newly defined misspecifiation condition and show that LinUCB is a no-regret algorithm in this case.

Model misspecification is naturally addressed in the related \emph{agnostic} contextual bandits setting \citep{agarwal2014taming}, but these approaches typically require the action space to be finite, thus not directly applicable to our problem. In addition, empirical evidence \citep{foster2018practical} suggests that the regression oracle approach works better in practice than the agnostic approach even if realizability cannot be verified. Recently, \citet{ye2023corruption} and \citet{ye2024corruption} studied corruption-robust nonlinear contextual bandits and offline reinforcement learning, respectively, but they still achieved $\tilde{O}(\sqrt{T}+\zeta)$ bound in the bandit setting where $\zeta$ is the total amount of corruption. Most recently, a concurrent work \citep{liu2024corruption} independently proposes a similar algorithm that works for gap-adjusted misspecified linear bandits, connects our gap-adjusted misspecification definition with corruption-robust learning, and extends it to reinforcement learning.

As a by-product, phased elimination enjoys another desired property known as ``deployment efficiency''. As it requires only $O(\log T)$ batches of explorations, each batch is embarrassingly parallel. Such low-adaptive (or batched) exploration problems are well-studied in bandits and reinforcement learning problems. The $\log(T)$-deployment efficiency from the current paper is a new state-of-the-art for linear bandits under misspecification. For well-specified problems, algorithms with $\tilde{O}(\sqrt{T})$ regret and $\log\log T$-deployment efficiency are known under various settings \citep{cesa2013online,perchet2016batched,gao2019batched,ruan2021linear,qiao2022sample,qiao2023logarithmic,zhang2022near}. It is unclear whether it can be achieved under misspecification. For $\tilde{O}(\log T)$ regret under the constant gap condition, our $\log T$-deployment efficiency is known to be near-optimal \citep{gao2019batched}.

\section{Preliminaries}\label{sec:pre}

\subsection{Notations}\label{sec:notation}
Let $[n]$ denote the integer set $\{1,2,...,n\}$ and $T$ denote the time horizon, i.e., total number of observations of the algorithm. Let $f_0$ denote the underlying true function, so the maximum function value is defined as $f^* = \max_{x \in \cX} f_0(x)$ and the maximum point is defined as $x^* = \argmax_{x \in \cX} f_0(x)$. Note $f_0$ can be a non-linear function since we are considering misspecificed bandits. Let $\cX \subseteq \R^d$ and $\cY \subseteq \R$ denote the domain and range of $f_0$. When $\mathcal{X}$ is finite, $|\mathcal{X}|=k$. Given a function $f$, $|f(x)-f_0(x)|$ denotes the approximation error and $f^* - f_0(x)$ denotes the suboptimality gap at point $x$. We use $\cW$ to denote the parameter class of a family of linear functions $\cF := \{f_w: \cX \rightarrow \cY|w \in \cW\}$ where $f_w(x)=w^\top x$. Let $w_* =\argmin_{w \in \cW} |f_w - f_0|$ denote the best approximation parameter.
Given a vector $x$, its $\ell_2$ norm is denoted by $\|x\|_2 = \sqrt{\sum_{i=1}^d x^2_i}$, given a matrix $A$ its operator norm is denoted by $\|A\|_\mathrm{op}$, and given a vector $x$ and a square matrix $A$, define $\|x\|^2_A = x^\top A x$.

\subsection{Problem Setup}\label{sec:setup}
We consider the following optimization problem:
\begin{align*}
x_* = \argmax_{x \in \cX} f_0(x),
\end{align*}
where $f_0$ is the true function which might not be linear in $\cX$. We want to use a linear function $f_w=w^\top x\in\mathcal{F}$ to approximate $f_0$ and maximize $f_0$. At time $0\leq t \leq T-1$, after querying a data point $x_t$, we will receive a noisy feedback:
\begin{align}
y_t = f_0(x_t) + \eta_t, \label{eq:obs}
\end{align}
where $\eta_t$ is independent, zero-mean, and $\sigma$-sub-Gaussian.

The major highlight of our study is that we do not rely on the popular \emph{realizability} assumption (\emph{i.e.} $f_0\in\mathcal{F}$) that is frequently assumed in the existing function approximation literature. Alternatively, we propose the following gap-adjusted misspecification condition.

\begin{definition}[$\rho$-Gap-Adjusted Misspecification]\label{def:lm}
We say a function $f$ is a $\rho$-gap-adjusted misspecified (or $\rho$-GAM in short) approximation of $f_0$ if 
for parameter $0 \leq \rho < 1$,
\begin{align*}
\sup_{x \in \cX} \left| \frac{f(x) - f_0(x)}{f^* - f_0(x)}\right|\leq \rho.\label{eq:local}
\end{align*}
We say function class $\cF=\{f_w | w\in\cW\}$ satisfies $\rho$-GAM in short) for $f_0$, if 
there exists $w^*\in\cW$ such that $f_{w_*}$ is a $\rho$-GAM approximation of $f_0$.
\end{definition}
Observe that when $\rho = 0$, this recovers the standard realizability assumption, but when $\rho>0$ it could cover many misspecified function classes.

Figure~\ref{fig:1d} shows a 1-dimensional example with $f_w(x)= 0.75x+0.5$ and piece-wise linear function $f_0(x)$ that satisfies local misspecification. With Definition~\ref{def:lm}, we have the following proposition. Its proof is shown in Appendix \ref{app:tech}.

\begin{proposition}\label{prop:perservation} Let $f$ be a $\rho$-GAM approximation of $f_0$ (Definition~\ref{def:lm}). Then it holds:
\begin{itemize}
\item (Preservation of maximizers) $$\argmax_{x}f(x) =\argmax_{x}f_{0}(x).$$
\item  (Preservation of max value) $$\max_{x\in\mathcal{X}}f(x)=f^*.$$
\item (Self-bounding property) $$|f(x) - f_0(x)| \leq \rho (f^* - f_0(x)) = \rho r(x).$$
\end{itemize}
\end{proposition}

This tells $f_{w_*}$ and $f_0$ coincide on the same global maximum points and the same global maxima if Definition \ref{def:lm} is satisfied, while allowing $f_{w_*}$ and $f_0$ to be different (potentially large) at other locations. Therefore, Definition~\ref{def:lm} is a ``local'' assumption that does not require $f_{w_*}$ to be uniformly close to $f_0$ (e.g. the ``uniform'' misspecification assumption $\sup_{x\in\mathcal{X}}|f_{w_*}(x)-f_0(x)|\leq \rho$).

\subsection{Assumptions}\label{sec:ass}
Here we list all assumptions that we work with throughout this paper. Note that the additional assumption is not required when $f_0$ is realizable. The first assumption is on the boundness of function domain, parameter space, and function range.

\begin{assumption}[Boundedness]\label{ass:bound}
For any $x\in\cX$, $\|x\|_2\leq C_b$.  For any $w\in\cW$, $\|w\|_2\leq C_w$. Moreover, for any $x,\tilde{x}\in\cX$, the true expected reward function $|f_0(x) - f_0(\tilde{x})| \leq 1$.
\end{assumption}
These are mild assumptions that we assume for convenience. Relaxations of these are possible but not the focus of this paper. 

\begin{assumption}\label{ass:one}
Suppose $\mathcal{X}\in\R^d$ is a compact set, and all the global maximizers of $f_0$ live on the $d-1$ dimensional hyperplane. i.e., $\exists a\in\R^d,b\in \R^1$, s.t. 
\begin{align*}
\argmax_{x\in\mathcal{X}}f_{0}(x)\subset \{x\in\R^d: x^\top a=b\}.
\end{align*}
\end{assumption}
For instance, when $d=1$, the above reduces to that $f_0$ has a unique maximizer. This is a compatibility assumption for Definition~\ref{def:lm}, since any linear function that violates Assumption~\ref{ass:one} will not satisfy Definition \ref{def:lm}.

In addition, we have the following two assumptions on gap-adjusted misspecification parameter $\rho$. Assumption \ref{ass:rho} is stronger than Assumption \ref{ass:rho2} and they will work with the LinUCB algorithm and the new PE algorithm respectively in this paper.

\begin{assumption}[Low Misspecification Error]\label{ass:rho}
The linear function class is a $\rho$-GAM approximation of $f_0$ with
\begin{align}
\rho < \frac{1}{8 d \sqrt{\log \left(1 + \frac{T C^2_b C^2_w}{d \sigma^2}\right)}} = O\left( \frac{1}{d\sqrt{\log T}}\right).
\end{align}
\end{assumption}
The condition is required for technical reasons with the LinUCB algorithm. While this assumption may suggest that we still require realizability in a truly asymptotic world since $T \rightarrow \infty, \rho =0$ violating the motivation of the $\rho$-GAM condition, handling a $O(1/\sqrt{\log T})$ level of misspecification is highly non-trivial in finite sample. For instance, if $T$ is a trillion, $1/\sqrt{\log (1e12)} \approx 0.19$. This means that for most practical cases, LinUCB is able to tolerate a constant level of misspecification under the GAM model. 

Relaxing this condition for LinUCB requires fundamental breakthroughs that knock out logarithmic factors from its regret analysis. This will be further clarified in the proof. Our last assumption is weaker than Assumption \ref{ass:rho} but our newly designed PE algorithm is still able to work with this assumption and achieve the $\tilde{O}((1+\rho)\sqrt{T})$ regret for a constant $\rho < 1$. The constant $16$ is due to technical reason which will be shown in proofs.

\begin{assumption}[Constant and Low Misspecification Error]\label{ass:rho2}
The linear function class is a $\rho$-GAM approximation of $f_0$ with
\begin{align*}
\rho \leq \frac{1}{16 \sqrt{d}}.
\end{align*}
\end{assumption}

\subsection{Algorithms}\label{sec:alg}

Here we first show details of the classical Linear Upper Confidence Bound (LinUCB) algorithm \citep{dani2008stochastic,abbasi2011improved} and then slightly modify the Phased Elimination (PE) algorithm from \citet{lattimore2020learning} by changing the constant from $2$ to $16$ in Step 6.

\begin{algorithm}[!htbp]
\caption{LinUCB \citep{abbasi2011improved}}
	\label{alg:linucb}
	{\bf Input:}
	Predefined sequence $\beta_t$ for $t=1,2,3,...$; 
 Set $\lambda=\sigma^2/C^2_w$ and $\mathrm{Ball}_0 = \cW$.
	\begin{algorithmic}[1]
	    \FOR{$t = 0,1,2,... $}
	    \STATE Select $x_t=\argmax_{x \in \cX} \max_{w \in \mathrm{Ball}_t} w^\top x$.
	    \STATE Observe $y_t = f_0(x_t) + \eta_t$.
     \STATE Update 
     \vspace{-1em}
     \begin{align}
\Sigma_{t+1} = \lambda I + \sum_{i=0}^{t} x_i x^\top_i \mathrm{where}\  \Sigma_0 = \lambda I.\label{eq:sigma_t}
\end{align}
\vspace{-1em}
	    \STATE Update 
	    \vspace{-1em}
	    \begin{align}
\hat{w}_{t+1} = \argmin_x \lambda \|w\|^2_2+ \sum_{i=0}^{t} (w^\top x_i - y_i)^2_2.\label{eq:w_t_opt}
\end{align}
\vspace{-1em}
    \STATE Update $\mathrm{Ball}_{t+1} = \{w | \|w - \hat{w}_{t+1}\|^2_{\Sigma_{t+1}} \leq \beta_{t+1}\}.$
		\ENDFOR
	\end{algorithmic}
\end{algorithm}

\begin{algorithm}[!htbp]
\caption{Phased Elimination (adapted from \citet{lattimore2020learning})}\label{alg:pe}
{\bf Input:}
$\cX \subseteq \mathbb{R}^d$, confidence level $\alpha \in (0, 1)$.
\begin{algorithmic}[1]
\STATE Set $m= \lceil 4d \log \log (d)\rceil +16$.
\STATE Find design $\pi: \cA \rightarrow [0,1]$ with $g(\pi) \leq 2d$ and $|\mathrm{supp}(\pi)|\leq 4 d \log \log(d) + 16$.
\STATE Compute $u(x) = \lceil m \pi(x)\rceil$ and $u=\sum_{x \in \cX} u(x)$.
\STATE Take each action $x \in \cX$ exactly $u(x)$ times with corresponding features $\{x_s\}_{s=1}^u$ and rewards $\{y_s\}_{s=1}^u$.
\STATE Calculate the vector $\hat{w}$:
\vspace{-1em}
\begin{align}
\hat{w} = G^{-1} \sum_{s=1}^u x_s y_s \quad \mathrm{where} \quad G=\sum_{s=1}^u x_s x_s^\top.\label{eq:hat_w}
\end{align}
\vspace{-1em}
\STATE Update the active set:
\vspace{-1em}
\begin{align*}
\cX \leftarrow \left\{ x \in \cX: \max_{b \in \cX} \hat{w}^\top (b -x) \leq 16 \sqrt{\frac{d}{m}\log \left(\frac{1}{\alpha}\right)} \right\}.
\end{align*}
\vspace{-1em}
\STATE $m \leftarrow 2m$ and GOTO Step 1.
\end{algorithmic}
\end{algorithm}

\section{Results of LinUCB Algorithm}\label{sec:linucb}

In this section, we show that the classical LinUCB algorithm \citep{abbasi2011improved} works in $\rho$-gap-adjusted misspecified linear bandits and achieves cumulative regret at the order of $\tilde{O}(\sqrt{T}/(1-\rho))$. The following theorem shows the cumulative regret bound.

\begin{theorem}
Suppose Assumptions \ref{ass:bound}, \ref{ass:one}, and \ref{ass:rho} hold. Set 
\begin{align}
\beta_t = 8\sigma^2 \left(1 + d\log\left(1+ \frac{t C^2_b C^2_w }{d \sigma^2} \right) + 2\log \left(\frac{\pi^2 t^2}{3\delta} \right)\right).\label{eq:beta_t}
\end{align} 
Then w.p. $> 1-\delta$ for simultaneously for all $T=1,2,...$
\begin{align*}
R_T &\leq 1 + \sqrt{\frac{8 (T-1) \beta_{T-1} d}{(1-\rho)^2} \log \left( 1 + \frac{T C^2_b C^2_w }{d \sigma^2 }\right)}.
\end{align*}
\end{theorem}
\begin{remark}
The cumulative regret bound shows that LinUCB achieves $\tilde{O}(\sqrt{T})$ cumulative regret bound and thus it is a no-regret algorithm in $\rho$-gap-adjusted misspecified linear bandits. In contrast, LinUCB can only achieve $\tilde{O}(\sqrt{T} + \epsilon T)$ regret in uniformly misspecified linear bandits. Even if $\epsilon = \tilde{O}(1/\sqrt{\log T})$, the resulting regret $\tilde{O}(T/\sqrt{\log T})$ is still exponentially worse than ours.
\end{remark}
\begin{proof}
By definition of cumulative regret, function range absolute bound $F$, and Cauchy-Schwarz inequality,
\begin{align*}
R_T &= r_0 + \sum_{t=1}^{T-1} r_t \\
&\leq 1 + \sqrt{\left(\sum_{t=1}^{T-1} 1 \right) \left(\sum_{t=1}^{T-1} r^2_t \right)}\\
&= 1 + \sqrt{ (T-1) \sum_{t=1}^{T-1} r^2_t}.
\end{align*}
Observe that the choice of $\beta_t$ is monotonically increasing in $t$. Also by Lemma~\ref{lem:w_t}, we get that with probability $1-\delta$, $w^*\in \text{Ball}_t \forall t= 1,2,3,...$, which verifies the condition to apply Lemma \ref{lem:sos_r_t} simultaneously for all $T=1,2,3,...$, thereby completing the proof.
\end{proof}

\subsection{Regret Analysis}\label{sec:reg_ana}

The proof follows the LinUCB analysis. The main innovation is a self-bounding argument that controls the regret due to misspecification by the regret itself.  This appears in Lemma~\ref{lem:r_t} and then again in the proof of Lemma~\ref{lem:w_t}.

Before we proceed, let $\Delta_t$ denote the deviation term of our linear function from the true function at $x_t$, formally,
\begin{align}
\Delta_t = f_0(x_t) - w^\top_* x_t,\label{eq:delta}
\end{align}
And our observation model (eq. \eqref{eq:obs}) becomes
\begin{align}
y_t = f_0(x_t) + \eta_t = w_*^\top x_t + \Delta_t + \eta_t.\label{eq:obs2}
\end{align}
Moreover, we have the following lemma showing the property of deviation term $\Delta_t$.
\begin{lemma}[Bound of Deviation]\label{lem:delta}
$\forall t \in \{0,1,\ldots,T-1\}$,
\begin{align*}
|\Delta_t | \leq \frac{\rho}{1-\rho} w^\top_*(x_* - x_t).
\end{align*}
\begin{proof}
Recall the definition of deviation term in eq. \eqref{eq:delta}:
\begin{align*}
\Delta_t = f_0(x_t) - w_*^\top x_t.
\end{align*}
By Definition \ref{def:lm}, $\forall t \in \{0,1,\ldots,T-1\}$,
\begin{align*}
-\rho(f^* - f_0(x_t))\leq \Delta_t &\leq \rho(f^* - f_0(x_t))\\
-\rho(f^* - w_*^\top x_t - \Delta_t)\leq \Delta_t &\leq \rho(f^* - w_*^\top x_t - \Delta_t)\\
-\rho(w_*^\top x_* - w_*^\top x_t - \Delta_t)\leq \Delta_t &\leq \rho(w_*^\top x_* - w_*^\top x_t - \Delta_t)\\
\frac{-\rho}{1-\rho} (w_*^\top x_* - w_*^\top x_t)\leq \Delta_t &\leq \frac{\rho}{1 + \rho}(w_*^\top x_* - w_*^\top x_t),
\end{align*}
where the third line is by Proposition \ref{prop:perservation} and the proof completes by taking the absolute value of the lower and upper bounds.
\end{proof}
\end{lemma}

Next, we prove instantaneous regret bound and its sum of squared regret version in the following two lemmas:

\begin{lemma}[Instantaneous Regret Bound]\label{lem:r_t}
Define $u_t := \| x_t\|_{\Sigma_t^{-1}}$, assume $w_*\in \mathrm{Ball}_t$
then for each $t\geq 1$
\begin{align*}
r_t \leq \frac{2\sqrt{\beta_t}u_t}{1-\rho}.
\end{align*}
\end{lemma}
\begin{proof}
By definition of instantaneous regret,
\begin{align*}
r_t &= f^* - f_0(x_t)\\
&= w^\top_* x_* - (w^\top_* x_t + \Delta(x_t))\\
&\leq w^\top_* x_* - w^\top_* x_t + \rho (f^* - f_0(x_t))\\
&= w^\top_* x_* - w^\top_* x_t + \rho r_t,
\end{align*}
where the inequality is by Definition \ref{def:lm}. Therefore, by rearranging the inequality we have
\begin{align*}
r_t &\leq \frac{1}{1-\rho}(w^\top_* x_* - w^\top_* x_t) \leq  \frac{2\sqrt{\beta_t} u_t}{1-\rho},
\end{align*}
where the last inequality is by Lemma \ref{lem:gap}.
\end{proof}

\begin{lemma}\label{lem:sos_r_t}
Assume $\beta_t$ is monotonically nondecreasing and $w_*\in \mathrm{Ball}_t$ for all $t=1,...,T-1$, then 
\begin{align*}
    \sum_{t=1}^{T-1} r^2_t \leq \frac{8\beta_{T-1} d}{(1-\rho)^2} \log \left( 1 + \frac{T C^2_b}{d \lambda }\right).
\end{align*}
\end{lemma}
\begin{proof}
By definition $u_t = \sqrt{x^\top_t \Sigma^{-1}_{t} x_t}$ and Lemma \ref{lem:r_t},
\begin{align*}
\sum_{t=1}^{T-1} r^2_t &\leq \sum_{t=1}^{T-1} \frac{4}{(1-\rho)^2} \beta_t u^2_t \\
&\leq \frac{4\beta_{T-1}}{(1-\rho)^2} \sum_{t=1}^{T-1} u^2_t \leq \frac{4\beta_{T-1}}{(1-\rho)^2} \sum_{t=0}^{T-1} u^2_t\\
&\leq \frac{8\beta_{T-1} d}{(1-\rho)^2} \log \left( 1 + \frac{T C^2_b}{d \lambda }\right),
\end{align*}
where the second inequality is by the monotonic increasing property of $\beta_t$ and the last inequality uses the elliptical potential lemma (Lemma \ref{lem:sum_pos}).
\end{proof}

Previous two lemmas hold on the following lemma, bounding the gap between $f^*$ and the linear function value at $x_t$, shown below.

\begin{lemma}\label{lem:gap}
Define $u_t = \| x_t\|_{\Sigma_t^{-1}}$ and assume $\beta_t$ is chosen such that $w_*\in \mathrm{Ball}_t$. 
Then
\begin{align*}
w_*^\top (x_* - x_t) \leq 2 \sqrt{\beta_t} u_t.
\end{align*}
\end{lemma}
\begin{proof}
Let $\tilde{w}$ denote the parameter that achieves $\argmax_{w \in \mathrm{Ball}_t} w^\top x_t$, by the optimality of $x_t$, 
\begin{align*}
w_*^\top x_* - w^\top_* x_t &\leq \tilde{w}^\top x_t - w^\top_* x_t \\
&= (\tilde{w} - \hat{w}_t + \hat{w}_t - w_*)^\top x_t\\
&\leq \|w_* - \hat{w}_t\|_{\Sigma_t} \|x_t\|_{\Sigma^{-1}_t} + \|\hat{w}_t - w_*\|_{\Sigma_t} \|x_t\|_{\Sigma^{-1}_t}\\
&\leq 2\sqrt{\beta_t} u_t
\end{align*}
where the second inequality applies Holder's inequality; the last line uses the definition of $\mathrm{Ball}_t$ (note that both $w_*,\tilde{w}\in \mathrm{Ball}_t).$
\end{proof}

\subsection{Confidence Analysis}\label{sec:conf_ana}
All analysis in the previous section requires $w_* \in \mathrm{Ball}_t, \forall t\in [T]$. In this section, we show that our choice of $\beta_t$ in \eqref{eq:beta_t} is valid and $w_*$ is trapped in the uncertainty set $\mathrm{Ball}_t$ with high probability.

\begin{lemma}[Feasibility of $\mathrm{Ball}_t$]\label{lem:w_t}
Suppose Assumptions \ref{ass:bound}, \ref{ass:one}, and \ref{ass:rho} hold. Set $\beta_t$ as in eq. \eqref{eq:beta_t}. Then, w.p. $> 1- \delta$,
\begin{align*}
\|w_* - \hat{w}_t\|^2_{\Sigma_t} \leq \beta_t, \forall t=1,2,...
\end{align*}
\end{lemma}
\begin{proof}
By setting the gradient of objective function in eq. \eqref{eq:w_t_opt} to be $0$, we obtain the closed form solution of eq. \eqref{eq:w_t_opt}:
\begin{align*}
\hat{w}_t = \Sigma_t^{-1} \sum_{i=0}^{t-1} y_i x_i.
\end{align*}
Therefore,
\begin{align}
\hat{w}_t - w_* &= - w_* + \Sigma_t^{-1} \sum_{i=0}^{t-1} x_i y_i \nonumber\\
&= - w_* + \Sigma_t^{-1} \sum_{i=0}^{t-1} x_i (x_i^\top w_* + \eta_i + \Delta_i) \nonumber\\
&= -w_* + \Sigma^{-1}_t \left(\sum_{i=0}^{t-1} x_i x_i^\top \right) w_* + \Sigma^{-1}_t \sum_{i=0}^{t-1} \eta_i x_i + \Sigma^{-1}_t \sum_{i=0}^{t-1} \Delta_i x_i,\label{eq:w_t_1}
\end{align}
where the second equation is by eq. \ref{eq:obs2} and the first two terms of eq. \eqref{eq:w_t_1} can be further simplified as
\begin{align*}
-w_* + \Sigma^{-1}_t \left(\sum_{i=0}^{t-1} x_i x_i^\top \right) w_* &= -w_* + \Sigma^{-1}_t \left(\lambda I + \sum_{i=0}^{t-1} x_i x_i^\top - \lambda I \right) w_*\\
&= - w_* + \Sigma_t^{-1} \Sigma_t w_* - \lambda \Sigma_t^{-1} w_*\\
& = - \lambda \Sigma^{-1}_t w_*,
\end{align*}
where the second equation is by definition of $\Sigma_t$ (eq. \eqref{eq:sigma_t}). Therefore, eq. \eqref{eq:w_t_1} can be rewritten as
\begin{align*}
\hat{w}_t - w_* = - \lambda \Sigma^{-1}_t w_*  + \Sigma^{-1}_t \sum_{i=0}^{t-1} \eta_i x_i  + \Sigma^{-1}_t \sum_{i=0}^{t-1} \Delta_i x_i.
\end{align*}
Multiply both sides by $\Sigma_t^{\half}$ and we have
\begin{align*}
\Sigma_t^{\half}(\hat{w}_t - w_*) &= - \lambda \Sigma^{-\half}_t w_* + \Sigma_t^{-\half} \sum_{i=0}^{t-1} \eta_i x_i + \Sigma^{-\half}_t \sum_{i=0}^{t-1} \Delta_i x_i.
\end{align*}
Take a square of both sides and apply generalized triangle inequality, we have
\begin{align}
\|\hat{w}_t - w_*\|^2_{\Sigma_t} & \leq 4 \lambda^2 \|w_*\|^2_{\Sigma_t^{-1}} + 4\left\| \sum_{i=0}^{t-1} \eta_i x_i \right\|^2_{\Sigma_t^{-1}} + 4\left\| \sum_{i=0}^{t-1} \Delta_i x_i \right\|^2_{\Sigma_t^{-1}}.\label{eq:w_t_2}
\end{align}
The remaining task is to bound these three terms separately. The first term of eq. \eqref{eq:w_t_2} is bounded as
\begin{align*}
4\lambda^2 \|w_*\|^2_{\Sigma^{-1}_t} \leq 4 \lambda \|w_*\|^2_2 \leq 4\sigma^2,
\end{align*}
where the first inequality is by definition of $\Sigma_t$ and $\|\Sigma^{-1}_t\|_\mathrm{op} \leq 1/\lambda$ and the second inequality is by choice of $\lambda = \sigma^2/C^2_w$.

The second term of eq. \eqref{eq:w_t_2} can be bounded by Lemma \ref{lem:self_norm} and Lemma \ref{lem:potential}:
\begin{align*}
4 \left\|\sum_{i=0}^{t-1} \eta_i x_i \right\|^2_{\Sigma_t^{-1}} &\leq 4\sigma^2 \log \left(\frac{\det (\Sigma_t) \det(\Sigma_0)^{-1}}{\delta_t^2} \right)\\
&\leq 4\sigma^2 \left(d \log\left(1 + \frac{t C^2_b}{d \lambda} \right) - \log \delta^2_t \right),
\end{align*}
where $\delta_t$ is chosen as $3\delta/(\pi^2 t^2)$ so that the total failure probabilities over $T$ rounds can always be bounded by $\delta/2$:
\begin{align*}
\sum_{t=1}^T \frac{3\delta}{\pi^2 t^2} < \sum_{t=1}^\infty \frac{3\delta}{\pi^2 t^2} = \frac{3\delta \pi^2 }{6 \pi^2} = \frac{\delta}{2}.
\end{align*}

And the third term of eq. \eqref{eq:w_t_2} can be bounded as
\begin{align*}
4 \left \| \sum_{i=0}^{t-1} \Delta_i x_i \right\|^2_{\Sigma^{-1}_t} &= 4\left(\sum_{i=0}^{t-1} \Delta_i x_i \right)^\top \Sigma^{-1}_t \left(\sum_{j=0}^{t-1} \Delta_j x_j \right)\\
&= 4 \sum_{i=0}^{t-1} \sum_{j=0}^{t-1} \Delta_i \Delta_j x_i \Sigma^{-1}_t x_j\\
&\leq 4\sum_{i=0}^{t-1} \sum_{j=0}^{t-1} |\Delta_i| |\Delta_j| \|x_i\|_{\Sigma^{-1}_t} \|x_j\|_{\Sigma^{-1}_t},
\end{align*}
where the last line is by taking the absolute value and Cauchy-Schwarz inequality. Continue the proof and we have
\begin{align*}
4\sum_{i=0}^{t-1} \sum_{j=0}^{t-1} |\Delta_i| |\Delta_j| \|x_i\|_{\Sigma^{-1}_t} \|x_j\|_{\Sigma^{-1}_t} &= 4\left( \sum_{i=0}^{t-1} |\Delta_i|  \|x_i\|_{\Sigma^{-1}_t}\right) \left(\sum_{j=0}^{t-1} |\Delta_j| \|x_j\|_{\Sigma^{-1}_t}\right)\\
&= 4\left( \sum_{i=0}^{t-1} |\Delta_i|  \|x_i\|_{\Sigma^{-1}_t}\right)^2\\
&\leq 4 \left(\sum_{i=0}^{t-1} |\Delta_i|^2 \right) \left(\sum_{i=0}^{t-1} \|x_j\|_{\Sigma^{-1}_t}^2 \right)\\
&\leq 4 d \rho^2 \sum_{i=0}^{t-1} r_i^2 .
\end{align*}
where the first inequality is due to Cauchy-Schwarz inequality and the second uses  the self-bounding properties $|\Delta_i| \leq \rho r_i$ from Proposition~\ref{prop:perservation} and Lemma~\ref{lem:sum_pos2}.

To put things together, we have shown that w.p. $> 1-\delta$, for any $t\geq 1$,
\begin{align}
\|\hat{w}_t-w_*\|^2_{\Sigma_t^{-1}} &\leq 4 \sigma^2 + 4\sigma^2 \left(d\log\left(1+ \frac{t C^2_b }{d \lambda} \right) + 2\log \left(\frac{\pi^2 t^2}{3\delta} \right)\right) + 4\rho^2 d \sum_{i=0}^{t-1} r_i^2, \label{eq:radius}
\end{align}
where we condition on \eqref{eq:radius} for the rest of the proof.

Observe that this implies that the feasibility of $w_*$ in $\mathrm{Ball}_t$ can be enforced if we choose $\beta_t$ to be larger than \eqref{eq:radius}. The feasiblity of $w_*$ in turn allows us to apply Lemma~\ref{lem:r_t} to bound the RHS with $\beta_{0},...,\beta_{t-1}$. We will use induction to prove that our choice 
$$\beta_t := 2\sigma^2\iota_t \text{ for } t=1,2,...$$ is valid, where short hand $$\iota_t:=4 + 4\left(d\log\left(1+ \frac{t C^2_b }{d \lambda} \right) + 2\log \left(\frac{\pi^2 t^2}{3\delta} \right)\right).$$

For the base case $t=1$, by eq. \eqref{eq:radius} and the definition of $\beta_1$ we directly have $\|\hat{w}_1-w_*\|^2_{\Sigma_1^{-1}}\leq \beta_1$. Assume our choice of $\beta_i$ is feasible for $i=1,...,t-1$, then we can write
\begin{align*}
    \|\hat{w}_t-w_*\|^2_{\Sigma_t^{-1}} &\leq \sigma^2\iota_t + 4\rho^2 d \sum_{i=1}^{t-1} \beta_i u_i^2 \\
    &\leq  \sigma^2\iota_t + 4\rho^2 d \beta_{t-1}\sum_{i=1}^{t-1} u_i^2,
    \end{align*}
where the second line is due to non-decreasing property of $\beta_t$. Then by Lemma \ref{lem:sum_pos} and Assumption~\ref{ass:rho}, we have
\begin{align}
\|\hat{w}_t-w_*\|^2_{\Sigma_t^{-1}}    &\leq \sigma^2\iota_t +8\rho^2 d^2 \beta_{t-1}\log \left(1+\frac{tC_b^2}{d\lambda} \right) \nonumber\\
    &\leq \sigma^2\iota_t + \half \beta_{t-1} \leq 2\sigma^2\iota_t = \beta_{t},
    \label{eq:known_rho}
\end{align}

The critical difference from the standard LinUCB analysis here is that if $\beta_{t-1}$ appears on the LHS of the bound and if its coefficient is larger, any valid bound for $\beta_t$ will have to grow exponentially in $t$. This is where Assumption \ref{ass:rho} helps us. Assumption \ref{ass:rho}  ensures that the coefficient of $\beta_{t-1}$ is smaller than $1/2$, so we can take $\beta_{t-1}\leq \beta_t$ and move $\beta_t/2$ to the right-hand side. 
\end{proof}

Proof of previous lemma needs Lemma \ref{lem:sum_pos2} and \ref{lem:sum_pos} in Appendix \ref{app:tech}.

\section{Results of Phased Elimination Algorithm}\label{sec:pe}

In this section, we present theoretical results of Algorithm \ref{alg:pe} on misspecified linear bandits under Assumption \ref{ass:rho2}, including the standard cumulative regret analysis in Section \ref{sec:main} and gap-dependent regret analysis in Section \ref{sec:gap}.

\subsection{Main Regret Analysis}\label{sec:main}

\begin{theorem}\label{thm:main}
Suppose Assumptions \ref{ass:bound}, \ref{ass:one}, \& \ref{ass:rho2} hold and $\alpha = 1/(kT)$. Then Algorithm \ref{alg:pe} guarantees $\forall \ T \geq 1$,
\begin{align*}
R_T &\leq O\left((1+ \rho)\sqrt{d T \log (kT)}\right).
\end{align*}
\end{theorem}
\begin{remark}
The cumulative regret bound shows that Algorithm \ref{alg:pe} achieves $\tilde{O}((1+\rho)\sqrt{T})$ regret and thus it is a no-regret algorithm under the $\rho$-GAM condition. In contrast, it achieves $\tilde{O}(\sqrt{T} +\epsilon T )$ regret under the uniform misspecification. Compared with the previous section, the improvement lies in Assumption \ref{ass:rho2} where $\rho$ can be a constant.
\end{remark}

\begin{proof}
Recall that Proposition \ref{prop:perservation} shows that if a function $f$ is the $\rho$-GAM approximation of $f_0$, it has the same maximizer and maximum function value as $f_0$. Also, the misspecification error can be upper bounded by suboptimality gap at $x \in \cX$.

And we rewrite the observation model as follows.
\begin{align}
y_t &= f_0(x_t) + \eta_t \nonumber\\
&= f(x_t) + \xi_t + \eta_t \nonumber \\
&= w^\top_* x_t + \xi_t + \eta_t,\label{eq:y_t}
\end{align}
where $\xi_t$ denotes the missspecification error and $f(x_t)$ is the linear approximation function.

Next, the proof has three steps.

\textbf{Step 1: Confidence analysis.}
According to the algorithm, $m_i = (\lceil 4d \log \log (d)\rceil +16)2^{i-1}$. In Step 1, we need to prove that after $m_i$, $x_*$ is not eliminated and $\forall x \in \cX$, 
\begin{align}
f^* - f_0(x) \leq 16\zeta \sqrt{\frac{d}{m_i}\log \left( \frac{1}{\alpha}\right)},\label{eq:goal_1}
\end{align}
where $\zeta$ is a constant which will be specified later.

First, we check if eq. \eqref{eq:goal_1} holds for $i=1$. By assumption \ref{ass:bound}, it means that
\begin{align*}
1 \leq 16\zeta \sqrt{\frac{d}{\lceil 4d \log \log (d)\rceil +16}\log \left( \frac{1}{\alpha}\right)},
\end{align*}
which requires
\begin{align}
\zeta \geq \frac{1}{16 \sqrt{\frac{d}{\lceil 4d \log \log (d)\rceil +16}\log \left( \frac{1}{\alpha}\right)}}. \quad \textbf{(Condition 1)} \label{eq:condition_1}
\end{align}
Next we assume the eq. \eqref{eq:goal_1} holds for all episodes $1,..., i-1$ then after episode $i$, $\forall b \in \cX$,
\begin{align}
|b^\top(\hat{w} - w_*)| &= \left | b^\top \left ( G^{-1} \sum_{s=1}^u x_s y_s \right) - b^\top w_*\right |\nonumber \\
&= \left | b^\top \left ( G^{-1} \sum_{s=1}^u x_s (x_s^\top w_* + \xi_s + \eta_s) \right) - b^\top w_*\right |\nonumber\\
&= \left | b^\top G^{-1} \sum_{s=1}^u x_s \xi_s + b^\top G^{-1} \sum_{s=1}^u x_s \eta_s \right |\nonumber \\
&\leq \left | b^\top G^{-1} \sum_{s=1}^u x_s \xi_s \right| + \left| b^\top G^{-1} \sum_{s=1}^u x_s \eta_s \right |,\label{eq:2_part}
\end{align}
where the first line is by estimation of $\hat{w}$ (eq. \eqref{eq:hat_w}), the second line is due to observation model (eq. \eqref{eq:y_t}, and the last line is by triangular inequality.

Then the first and second terms in eq. \eqref{eq:2_part} need to be bounded separately. By extracting the misspecification term out and taking the maximum of it over all $x \in \cX$, the first term of eq. \eqref{eq:2_part} is bounded as follows.
\begin{align}
\left | b^\top G^{-1} \sum_{s=1}^u x_s \xi_s \right| &\leq \max_{x \in \cX} |\xi_x| \sum_{s =1}^u |b^\top G^{-1} x_s| \nonumber\\
&\leq \max_{x \in \cX} |\xi_x| \sqrt{\sum_{s =1}^u b^\top \sum_{s'=1}^u G^{-1} x_s x^\top_{s'} G^{-1} b } \nonumber\\
&= \max_{x \in \cX} |\xi_x| \sqrt{\sum_{s =1}^u \|b\|^2_{G^{-1}}} \nonumber\\
&\leq \max_{x \in \cX} |\xi_x| \sqrt{\frac{2du}{m_{i}}} \nonumber\\
&\leq \sqrt{2d} \max_{x \in \cX} |\xi_x|, \label{eq:tmp_1}
\end{align}
where the second inequality is due to Jensen's inequality, the second last inequality is by the property of $G$ such that $\|b\|^2_{G^{-1}} \leq 2d/m_i$, and the last inequality is due to calculation of $u(x)$ in Algorithm \ref{alg:pe}. By Proposition \ref{prop:perservation} and assumption in eq. \eqref{eq:goal_1}, we have
\begin{align*}
    |\xi_x| \leq \rho(f^* - f_0(x)) \leq 16 \rho \zeta \sqrt{\frac{2d}{m_i} \log \left(\frac{1}{\alpha}\right)}.
\end{align*}
So eq. \eqref{eq:tmp_1} can be further upper bounded as
\begin{align}
\sqrt{2d} \max_{x \in \cX} |\xi_x| &\leq 16 \sqrt{2d} \rho \zeta \sqrt{\frac{2d}{m_i} \log \left(\frac{1}{\alpha}\right)} = 32 d \rho \zeta \sqrt{\frac{1}{m_i} \log \left(\frac{1}{\alpha}\right)}, \label{eq:part_1}
\end{align}

And using eq. (20.2) of \citet{lattimore2020bandit}, the second term of eq. \eqref{eq:2_part} is bounded with probability $ > 1-2\alpha$,
\begin{align}
\left| b^\top G^{-1} \sum_{s=1}^u x_s \eta_s \right | \leq 2 \sqrt{\frac{d}{m_i}\log \left( \frac{1}{\alpha}\right)}.\label{eq:part_2}
\end{align}

Therefore, combine eq. \eqref{eq:part_1} and \eqref{eq:part_2} together and we have
\begin{align}
|b^\top (\hat{w} - w_*)| \leq (2 + 32\sqrt{d} \rho \zeta) \sqrt{\frac{d}{m_i}\log \left( \frac{1}{\alpha}\right)}. \label{eq:confidence}
\end{align}

\textbf{Step 2: Suboptimality upper bound.}

Let $\hat{x}= \argmax_{x \in \cX} \hat{w}^\top x$, then
\begin{align}
\max_{b \in \cX} \hat{w}^\top(b-x_*) &= \hat{w}^\top(\hat{x} - x_*) \nonumber\\
&\leq w^\top_* (\hat{x} - x_*) + (4 + 64\sqrt{d} \rho \zeta)\sqrt{\frac{d}{m}\log\left(\frac{1}{\alpha}\right)} \nonumber\\
&\leq (4 + 64\sqrt{d} \rho \zeta)\sqrt{\frac{d}{m}\log\left(\frac{1}{\alpha}\right)}, \label{eq:condition_2_pre}
\end{align}
where the second inequality is by using eq. \eqref{eq:confidence} twice and the last inequality is due to property of $w_*$ and $x_*$. Compared with Step 6 of Algorithm \ref{alg:pe}, note eq. \eqref{eq:condition_2_pre} requires
\begin{align}
4 + 64\sqrt{d} \rho \zeta \leq 16. \quad \textbf{(Condition 2)}\label{eq:condition_2}
\end{align}

If $x$ is not eliminated after $m_i$ episodes, i.e.,
\begin{align*}
16\sqrt{\frac{d}{m_i}\log \left( \frac{1}{\alpha}\right)} & \geq \max_{b \in \cX} \hat{w}^\top (b-x)\\
&\geq \hat{w}^\top (x_* - x)\\
&\geq w^\top_* (x_* - x) - (4 + 64\sqrt{d} \rho \zeta)\sqrt{\frac{d}{m_i}\log\left(\frac{1}{\alpha}\right)}\\
&\geq f^* - w^\top_* x - 16\sqrt{\frac{d}{m_i}\log\left(\frac{1}{\alpha}\right)}\\
&\geq f^* - f_0(x) - 16\sqrt{\frac{d}{m_i}\log\left(\frac{1}{\alpha}\right)} - 16\rho \zeta \sqrt{\frac{2d}{m_i}\log\left(\frac{1}{\alpha}\right)},
\end{align*}
where the third inequality again uses eq. \eqref{eq:confidence} twice, the fourth inequality is due to \textbf{Condition 2}, and the last inequality is using
\begin{align*}
|f_0(x) - w^\top_* x| \leq \rho (f^* - f_0(x)) \leq 16 \rho \zeta \sqrt{\frac{2d}{m_i}\log\left(\frac{1}{\alpha}\right)}.
\end{align*}
After arranging the result, we have
\begin{align}
f^* - f_0(x) \leq (32 + 16\sqrt{2} \rho \zeta)\sqrt{\frac{2d}{m_i}\log\left(\frac{1}{\alpha}\right)},\label{eq:subopt}
\end{align}
which requires
\begin{align}
    32 + 16\sqrt{2} \rho \zeta \leq 16\zeta. \quad \textbf{(Condition 3)}\label{eq:condition_3}
\end{align}

By considering \textbf{Conditions 1, 2, \& 3} (eq. \eqref{eq:condition_1}, \eqref{eq:condition_2}, \& \eqref{eq:condition_3}) together, it suffices to choose
\begin{align*}
\zeta = 3 \quad \mathrm{and} \quad \rho \leq \frac{1}{16\sqrt{d}}.
\end{align*}

\textbf{Step 3. Combine the episodes.} The last step is to combine all episodes together to prove the main cumulative regret bound. By definition of cumulative regret (eq. \eqref{eq:cr}),
\begin{align*}
R_T &= m_1 + \sum_{i=2}^L \sum_{s=1}^u (f^* - f_0(x_s))\\
&\leq m_1 + \sum_{i=2}^L m_i (32 + 48\sqrt{2} \rho) \sqrt{\frac{4d}{m_i} \log \left(\frac{1}{\alpha}\right)} \\
&\leq O\left( (1 + \rho) \sqrt{d m_L \log \left(\frac{1}{\alpha}\right)}\right)\\
&\leq O \left( (1 + \rho) \sqrt{d T \log \left(\frac{1}{\alpha}\right)} \right),
\end{align*}
where the second line is by eq. \eqref{eq:subopt} and the last line is due to $L = O(\log T)$.
\end{proof}

\subsection{Gap-Dependent Analysis}\label{sec:gap}
In this part, we provide gap-dependence regret analysis for Algorithm \ref{alg:pe}. We assume that  for all suboptimal action $x\in \cX \backslash \{x^*\}$,
\begin{align*}
f_0(x_*) - f_0(x) \geq \Delta.
\end{align*}
Then we can state and prove the following theorem.

\begin{theorem}[Gap-Dependent Cumulative Regret]\label{thm:gap-dep}
Suppose Assumptions \ref{ass:bound}, \ref{ass:one}, \& \ref{ass:rho2} hold and $\alpha = 1/(kT)$. Then Algorithm \ref{alg:pe} guarantees $\forall \ T \geq 1$,
\begin{align*}
R_T &\leq O\left(\frac{d \log kT}{\Delta}\right).
\end{align*}
\end{theorem}
\begin{remark}
If a positive suboptimality gap $\Delta$ exists, the cumulative regret bound shows that Algorithm \ref{alg:pe} achieves $\tilde{O}(\log T / \Delta)$ regret. Therefore, Algorithm \ref{alg:pe} is still a no-regret algorithm under the $\rho$-GAM condition and the $\tilde{O}(\log T)$ regret is better than the $\tilde{O}(\sqrt{T})$ regret in Theorem \ref{thm:main}.
\end{remark}

\begin{proof}
Assume the same high-probability events in \eqref{eq:goal_1} hold, which means for all $i$, after $m_i = (\lceil 4d \log \log (d)\rceil +16)2^{i-1}$ episodes, suboptimality of uneliminated arms is bounded by $c_1 \sqrt{\frac{d\log kT}{m_i}}$ for some constant $c_1$.

Under the above high-probability events, let $i_0$ denote the last batch where there may be positive regret, i.e, $i_0$ is the smallest integer such that
$$c_1 \sqrt{\frac{d \log kT}{(\lceil 4d \log \log (d)\rceil +16)2^{i_0-1}}}\leq \Delta,$$
which implies that for some $c_2>0$,
\begin{equation*}
   2^{i_0} \leq \frac{c_2 \log kT}{\Delta^2}.
\end{equation*}

Then similar to analysis in Theorem \ref{thm:main}, cumulative regret can be bounded as
\begin{align*}
R_T &\leq m_1 + \sum_{i=1}^{i_0 - 1} m_{i+1}\cdot c_1 \sqrt{\frac{d \log kT}{m_i}}\\
&\leq d + \sum_{i=1}^{i_0 -1 } 2 c_1 \sqrt{dm_i\cdot\log kT}\\
&\leq c_3 \sqrt{dm_{i_0}\cdot\log kT}\\
&\leq c_4 \frac{d \log kT}{\Delta},
\end{align*}
where $c_3,c_4$ are universal constants.
\end{proof}

\section{Conclusion}\label{sec:conclusion}
Linear stochastic bandits are classical problems in online learning that can be used in many applications, including A/B testing, recommendation system, clinical trial optimization, and materials design. 
We study linear bandits with the underlying reward function being non-linear, which falls into the misspecified bandit framework. Existing work on misspecified bandit usually assumes uniform misspecification where the $\ell_\infty$ distance between the best-in-class function and the true function is upper bounded by the misspecification parameter $\epsilon$. Existing lower bound shows that the $\tilde{\Omega}(\epsilon T)$ term is unavoidable where $T$ is the time horizon, thus the regret bound is always linear. However, in solving optimization problems, one only cares about the approximation error near the global optimal point and approximation error is allowed to be large in highly suboptimal regions. In this paper, we capture this intuition and define a natural model of misspecification, called $\rho$-gap-adjusted misspecificaiton, which only requires the approximation error at each input $x$ to be proportional to the suboptimality gap at $x$ with $\rho$ being the proportion parameter. 

Previous work found that classical LinUCB algorithm is not robust in $\epsilon$-uniform misspecified linear bandit when $\epsilon$ is large. However, we show that LinUCB is automatically robust against such gap-adjusted misspecification. Under mild conditions, e.g., $\rho \leq O(1/\sqrt{\log T})$, we prove that it achieves the near-optimal $\tilde{O}(\sqrt{T})$ regret for problems that the best-known regret is almost linear. Also, LinUCB doesn't need the knowledge of $\rho$ to run. However, if the upper bound of $\rho$ is revealed to LinUCB, the $\beta_t$ term can be carefully chosen according to eq. \eqref{eq:known_rho}. Our technical novelty lies in a new self-bounding argument that bounds part of the regret due to misspecification by the regret itself, which can be of independent interest in more settings.

We further advance the frontier of the GAM condition by presenting a novel Phased Elimination-based (PE) algorithm. We prove that for a fixed $\rho = O(1/\sqrt{d})$, the new algorithm achieves optimal $\tilde{O}(\sqrt{T})$ cumulative regret. Surprisingly, as a by-product, the PE algorithm requires only $O(\log T)$ policy switching cost thanks to the phased elimination algorithmic design, which is highly deployment-efficient. It also enjoys an adaptive $\tilde{O}(\log T)$ regret when a constant suboptimality gap exists.

More broadly, our paper opens a brand new door for research in model misspecification, including misspecified linear bandits, misspecified kernelized bandits, and even reinforcement learning with misspecified function approximation. Moreover, we hope our paper make people rethink about the relationship between function optimization and function approximation. 

\acks{The work was partially supported by UAlbany Computer Science Department startup, NSF Awards \#2007117 and \#2003257, EPSRC New Investigator Award EP/X03917X/1, and the Engineering and Physical Sciences Research Council EP/S021566/1.}

\bibliography{bib}

\newpage
\appendix

\section{Technical Lemmas}\label{app:tech}

\begin{lemma}[Self-Normalized Vector-Valued Martingales (Lemma A.9 of  \citet{agarwal2021rl})]\label{lem:self_norm}
Let $\{\eta_i\}_{i=1}^\infty$ be a real-valued stochastic process with corresponding filtration $\{\cF_i\}_{i=1}^\infty$ such that $\eta_i$ is $\cF_i$ measurable, $\E[\eta_i | \cF_{i-1}] = 0$, and $\eta_i$ is conditionally $\sigma$-sub-Gaussian with $\sigma \in \R^+$. Let $\{ X_i \}_{i=1}^\infty$ be a stochastic process with $X_i \in \cH$ (some Hilbert space) and $X_i$ being $\cF_t$ measurable. Assume that a linear operator $\Sigma: \cH \rightarrow \cH$ is positive deﬁnite, i.e., $x^\top \Sigma x > 0$ for any $x \in \cH$. For any $t$, define the linear operator $\Sigma_t = \Sigma_0 + \sum_{i=1}^t X_i X^\top_i$ (here $xx^\top$ denotes outer-product in $\cH$). With probability at least $1-\delta$, we have for all $t \geq 1$:
\begin{align*}
\left\| \sum_{i=1}^t X_i \eta_i \right\|^2_{\Sigma_t^{-1}} \leq \sigma^2 \log \left( \frac{\det(\Sigma_t) \det(\Sigma_0)^{-1}}{\delta^2}\right).  
\end{align*}
\end{lemma}

\begin{lemma}[Sherman-Morrison Lemma \citep{sherman1950adjustment}]\label{lem:sm}
Let $A$ denote a matrix and $b,c$ denote two vectors. Then
\begin{align*}
(A + bc^\top)^{-1} = A^{-1} - \frac{A^{-1} bc^\top A^{-1}}{1+ c^\top A^{-1} b}.
\end{align*}
\end{lemma}

\begin{lemma}[Lemma 6.10 of \citet{agarwal2021rl}]\label{lem:det}
Define $u_t = \sqrt{x^\top_t \Sigma^{-1}_t x_t}$ and we have
\begin{align*}
\det \Sigma_T = \det \Sigma_0 \prod_{t=0}^{T-1} (1 + u^2_t). 
\end{align*}
\end{lemma}

\begin{lemma}[Potential Function Bound (Lemma 6.11 of \citet{agarwal2021rl})]\label{lem:potential}
For any sequence $x_0,...,x_{T-1}$ such that for $t< T, \|x_t\|_2 \leq C_b$, we have
\begin{align*}
\log \left( \frac{\det \Sigma_{T-1}}{\det \Sigma_0}\right) &= \log \det \left( I + \frac{1}{\lambda} \sum_{t=0}^{T-1} x_t x^\top_t \right)\nonumber\\
& \leq d\log\left(1+ \frac{TC_b^2}{d \lambda} \right). 
\end{align*}
\end{lemma}

\begin{lemma}[Upper bound of $\sum_{i=0}^{t-1} x^\top_i \Sigma_t^{-1} x_i$]\label{lem:sum_pos2}
\begin{align*}
\sum_{i=0}^{t-1} x^\top_i \Sigma^{-1}_t x_i \leq d.
\end{align*}
\end{lemma}
\begin{proof}
Recall that $\Sigma_t = \sum_{i=0}^{t-1} x_i x_i^T + \lambda I_d$.
\begin{align*} \sum_{i=0}^{t-1} x^\top_i \Sigma^{-1}_t x_i   &= \sum_{i=0}^{t-1}\mathrm{tr}\left[ 
 \Sigma^{-1}_t x_ix_i^T \right]\\
 &= \mathrm{tr}\left[ 
 \Sigma^{-1}_t \sum_{i=0}^{t-1} x_ix_i^T \right] \\
 &= \mathrm{tr}\left[ 
 \Sigma^{-1}_t (\Sigma_t - \lambda I_d)\right]  \\
 &= \mathrm{tr}\left[I_d\right] - \mathrm{tr}\left[\lambda \Sigma^{-1}_t\right]\leq d.
 \end{align*}
 The last line follows from the fact that $\Sigma^{-1}_t$ is positive semidefinite.
\end{proof}

\begin{lemma}[Upper bound of $\sum_{i=0}^{t-1} x^\top_i \Sigma_i^{-1} x_i$ (adapted from \citet{abbasi2011improved})]\label{lem:sum_pos}
\begin{align*}
\sum_{i=0}^{t-1} x^\top_i \Sigma^{-1}_i x_i \leq 2d \log \left(1 + \frac{t C_b^2}{d \lambda} \right).
\end{align*}
\end{lemma}

\begin{proof}
First we prove that $\forall i \in \{0, 1,..., t-1\}, 0\leq x_i^\top \Sigma^{-1}_i x_i < 1$. Recall the definition of $\Sigma_i$ and we know $\Sigma^{-1}_i$ is a positive semidefinite matrix and thus $0 \leq x_i^\top \Sigma^{-1}_i x_i$. To prove $x_i^\top \Sigma^{-1}_i x_i < 1$, we need to decompose $\Sigma_i$ and write
\begin{align*}
x_i^\top \Sigma^{-1}_i x_i &= x_i^\top \left(\lambda I + \sum_{j=0}^{i-1} x_j x^\top_j \right)^{-1} x_i\\
&= x_i^\top \left(x_i x_i^\top - x_i x_i^\top + \lambda I + \sum_{j=0}^{i-1} x_j x^\top_j \right)^{-1} x_i.
\end{align*}
Let $A = - x_i x_i^\top + \lambda I + \sum_{j=0}^{i-1} x_j x^\top_j$ and it becomes
\begin{align*}
x^\top_i \Sigma^{-1}_i x_i = x^\top_i (x_i x^\top_i + A)^{-1} x_i.
\end{align*}
By Sherman-Morrison lemma (Lemma \ref{lem:sm}), we have
\begin{align*}
x^\top_i \Sigma^{-1}_i x_i &= x^\top_i \left(A^{-1} - \frac{A^{-1} x_i x^\top_i A^{-1}}{1 + x^\top_i A^{-1} x_i} \right) x_i\\
&= x^\top_i A^{-1} x_i - \frac{x^\top_i A^{-1} x_i x^\top_i A^{-1} x_i}{1 + x^\top_i A^{-1} x_i}\\
&= \frac{x^\top_i A^{-1} x_i}{1 + x^\top_i A^{-1} x_i} < 1.
\end{align*}
Next we use the fact that $\forall x \in [0, 1), x \leq 2\log(x+1)$ and we have
\begin{align*}
\sum_{i=0}^{t-1} x^\top_i \Sigma^{-1}_i x_i &\leq \sum_{i=0}^{t-1} 2\log \left(1+ x^\top_i \Sigma^{-1}_i x_i \right)\\
&\leq 2 \log \left( \frac{\det(\Sigma_{t-1})}{\det(\Sigma_0)} \right)\\
&\leq 2 d \log \left( 1 + \frac{t C^2_b}{d \lambda}\right),
\end{align*}
where the last two lines are by Lemma \ref{lem:det} and Lemma \ref{lem:potential}.
\end{proof}

Finally, we restate Proposition \ref{prop:perservation} and show its proofs.

\begin{proposition}[Restatement of Proposition \ref{prop:perservation}]
Let $f$ be a $\rho$-GAM approximation of $f_0$ (Definition~\ref{def:lm}). Then it holds:
\begin{itemize}
\item (Preservation of maximizers) $$\argmax_{x}f(x) =\argmax_{x}f_{0}(x).$$
\item  (Preservation of max value) $$\max_{x\in\mathcal{X}}f(x)=f^*.$$
\item (Self-bounding property) $$|f(x) - f_0(x)| \leq \rho (f^* - f_0(x)) = \rho r(x).$$
\end{itemize}
\end{proposition}

Equivalently, $\rho$-gap-adjusted misspecification (Definition \ref{def:lm}) satisfies 
\begin{equation}\label{eqn:rho_miss}
 \left|f(x) - f_0(x) \right|\leq \rho    
\left|f^* - f_0(x)\right|,\;\;\forall x \in \cX.
\end{equation}

\begin{proof}[Proof of preservation of max value: $\max_{x\in\mathcal{X}}f(x)=f^*$]

Let $f^*_w := \max_{x\in\mathcal{X}}f(x)$. We first prove $f^*_w\leq f^*$ by contradiction. Suppose $f^*_w > f^*$, since $\mathcal{X}$ is compact, there exists $x_w\in\cX$ such that $f(x_w)=f^*_w>f^*$. Then by eq. \eqref{eqn:rho_miss} this implies
\[
f(x_w)-f_0(x_w)\leq \rho (f^*-f_0(x_w))\Rightarrow f^*<f^*_w=f(x_w)\leq \rho f^*+(1-\rho)f_0(x_w)\leq f^*
\]
Contraction! Therefore, $f_w^*\leq f^*$. On the other hand, choose $x_0\in\argmax_{x\in\cX}f_0(x)$, then by \eqref{eqn:rho_miss} $f(x_0)=f_0(x_0)=f^*$. This implies $f_w^*\geq f^*$. Combing both results to obtain $f_w^*= f^*$.
\end{proof}

\begin{proof}[Proof of preservation of maximizers: $\argmax_{x}f(x) =\argmax_{x}f_{0}(x)$]

Using that $f(x)\leq \rho f^*+(1-\rho)f_0(x)$ and $\max_{x\in\mathcal{X}}f(x)=f^*$, it is easy to verify $\argmax_{x}f(x) \subset\argmax_{x}f_{0}(x)$. On the other hand, if $x'\in\argmax_{x}f_{0}(x)$, then by eq. \eqref{eqn:rho_miss} $f(x')=f_0(x')=f^*$ and this means $\argmax_{x}f_0(x) \subset\argmax_{x}f(x)$. 
\end{proof}

\begin{proof}[Proof of self-bounding property]
This directly comes from the definition.
\end{proof}

\section{Weak Gap-Adjusted Misspecification}\label{app:weak}

In addition, we can modify Definition~\ref{def:lm} with a slightly weaker condition that only requires $\argmax_{x}f_{w_*}(x) =\argmax_{x}f_{0}(x)$
but not necessarily $\max_{x\in\mathcal{X}}f_{w_*}(x)=f^*$.

\begin{definition}[Weaker $\rho$-gap-adjusted misspecification]\label{def:lm_weak}
Denote $f_w^*=\max_{x\in\mathcal{X}} f_w(x)$. There exists $w\in\cW$ such that for a parameter $0 \leq \rho < 1$,
\begin{align*}
\sup_{x \in \cX} \left| \frac{f_{w}(x) - f^*_w+f^*-f_0(x)}{f^* - f_0(x)}\right|\leq \rho.
\end{align*}
\end{definition}

\begin{remark}
See Figure \ref{fig:weak} for an example satisfying Definition \ref{def:lm_weak}. Both Definition~\ref{def:lm} and Definition~\ref{def:lm_weak} are defined in the generic way that does not require any assumption on the parametric form of $f_w$. While in this paper we focus on the linear bandit setting, this notion can be applied to arbitrary parametric function approximation learning problem. In this paper, we stick to Definition~\ref{def:lm} and linear function approximation for conciseness and clarity.
\end{remark}

\begin{figure*}[!htbp]
	\centering    
	\subfigure[$\rho$-gap-adjusted misspecification]{\label{fig:example}\includegraphics[width=0.45\linewidth]{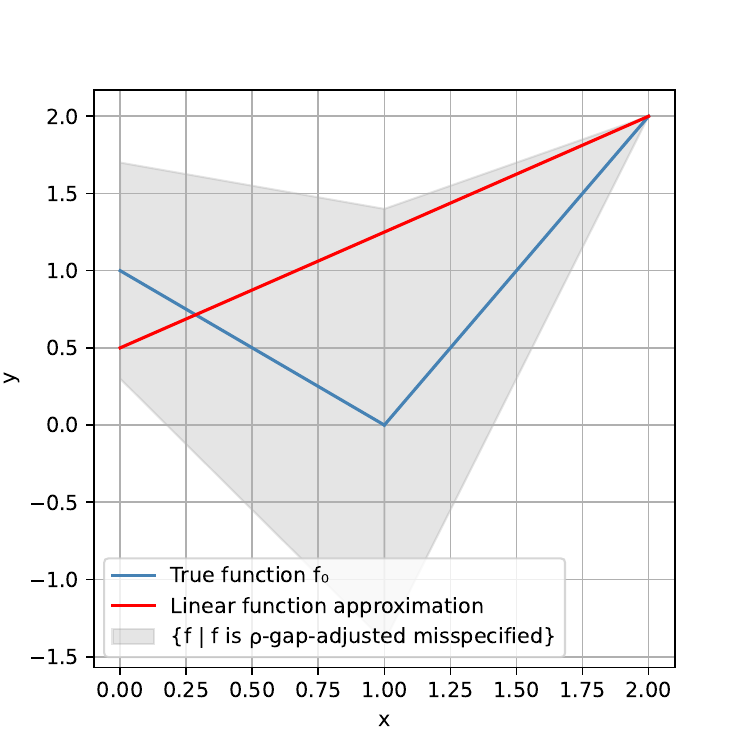}}
	\subfigure[Weak $\rho$-gap-adjusted misspecification]{\label{fig:example2}\includegraphics[width=0.45\linewidth]{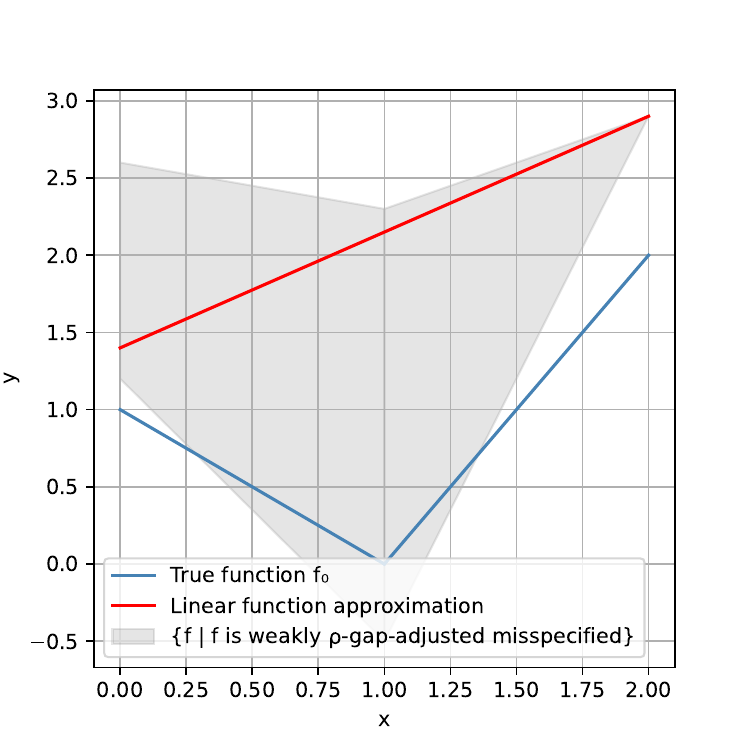}}
	\caption{(a): An example of $\rho$-gap-adjusted misspecification (Definition \ref{def:lm}) in $1$-dimension where $\rho=0.7$. The blue line shows a non-linear true function and the gray region shows the gap-adjusted misspecified function class. Note the vertical range of gray region at a certain point $x$ depends on the suboptimal gap. For example, at $x=1$ suboptimal gap is $2$ and the vertical range is $4\rho=2.8$. The red line shows a feasible linear function that is able to optimize the true function by taking $x_*=2$.  (b): An example of weak $\rho$-gap-adjusted misspecification (Definition \ref{def:lm_weak}) in $1$-dimension where $\rho=0.7$. The difference to Figure \ref{fig:example} is that one can shift the qualifying approximation arbitrarily up or down and the specified model only has to $\rho$-RAM approximate $f_0$ up to an additive constant factor.}
	\label{fig:weak}
\end{figure*}

\subsection{Property of Weak Gap-Adjusted Misspecification}

Under the weak $\rho$-gap-adjusted misspecification condition, it no longer holds $f_w^*=f^*$. However, it still preserves the maximizers.

\begin{proposition}\label{prop_weka_rho}
 Under the weak $\rho$-gap-adjusted misspecification condition, it holds $$\argmax_{x}f(x) =\argmax_{x}f_{0}(x).$$
\end{proposition}
\begin{proof}
Suppose $x'\in\argmax_{x}f(x)$, then by definition
\[
|f^*-f_0(x')|=|f(x')-f_w^*+f^*-f_0(x')|\leq \rho |f^*-f_0(x')|\Rightarrow (1-\rho) |f^*-f_0(x')|\leq 0\Rightarrow x'\in\argmax_{x}f_0(x).
\]
On the other hand, if $x'\in\argmax_{x}f_0(x)$, then
\[
|f_w^*-f(x')|=|f(x')-f_w^*+f^*-f_0(x')|\leq \rho |f^*-f_0(x')|=0\Rightarrow x'\in\argmax_{x}f(x). 
\]
\end{proof}

The next proposition shows the weak $\rho$-adjusted misspecification condition characterizes the suboptimality gap between $f$ and $f_0$.

\begin{proposition}
    Denote $g(x):= f^*_w-f(x)\geq 0$, $g_0(x):=f^*-f_0(x)\geq 0$, then the weak $\rho$-gap-adjusted misspecification condition implies:
    \[
    (1-\rho)g_0(x)\leq g(x)\leq (1+\rho) g_0(x),\quad x\in\cX.
    \]
\end{proposition}
This can be proved directly by the triangular inequality. This reveals the weak $\rho$-gap-adjusted misspecification condition requires $g(x)$ to live in the band $[(1-\rho)g_0(x),(1+\rho) g_0(x)]$, and the concrete maximum values $f_w^*$ and $f^*$ can be arbitrarily different. 

\subsection{Linear Bandits under the Weak Gap-Adjusted Misspecification}\label{sec:weak_regret}

We need to slightly modify LinUCB \citep{abbasi2011improved} and work with the following LinUCBw algorithm.

\begin{algorithm}[!htbp]
\caption{LinUCBw (adapted from \citet{abbasi2011improved})}
	\label{alg:linucb2}
	{\bf Input:}
	Predefined sequence $\beta_t$ for $t=1,2,3,...$ as in eq. \eqref{eq:beta_t_2};
 Set $\lambda=\sigma^2/C^2_w$ and $\mathrm{Ball}_0 = \cW$.
	\begin{algorithmic}[1]
	    \FOR{$t = 0,1,2,... $}
	    \STATE Select $x_t=\argmax_{x \in \cX} \max_{[w^\top,c] \in \mathrm{Ball}_t} [w^\top,c] \begin{bmatrix}x\\1\end{bmatrix}$.
	    \STATE Observe $y_t = f_0(x_t) + \eta_t$.
     \STATE Update 
     \begin{align*}
\Sigma_{t+1} = \lambda I_{d+1} + \sum_{i=0}^{t} \begin{bmatrix}x_i\\1\end{bmatrix} \cdot [x^\top_i,1] \ \mathrm{where}\  \Sigma_0 = \lambda I_{d+1}.
\end{align*}
	    \STATE Update 
	    \begin{align*}
\begin{bmatrix}\hat{w}_{t+1}\\\hat{c}_{t+1}\end{bmatrix} = \argmin_{w,c} \lambda \left \|\begin{bmatrix}w\\c\end{bmatrix} \right\|^2_2+ \sum_{i=0}^{t} (w^\top x_i +c- y_i)^2_2.
\end{align*}
    \STATE Update
	    \begin{align*}
     \mathrm{Ball}_{t+1} = \left \{
    \begin{bmatrix}w\\c\end{bmatrix} \bigg\rvert \left\|\begin{bmatrix}w\\c\end{bmatrix} - \begin{bmatrix}\hat{w}_{t+1}\\\hat{c}_{t+1}\end{bmatrix} \right\|^2_{\Sigma_{t+1}} \leq \beta_{t+1} \right\}.
    \end{align*}
		\ENDFOR
	\end{algorithmic}
\end{algorithm}

\begin{theorem}\label{thm:2}
Suppose Assumptions \ref{ass:bound}, \ref{ass:one}, and \ref{ass:rho} hold. W.l.o.g., assuming $c^*=f^*-f_w^*\leq F$. Set 
\begin{align}
\beta_t = 8\sigma^2 \left(1 + (d+1)\log\left(1+ \frac{t C^2_b (C^2_w+F^2) }{d \sigma^2} \right) + 2\log \left(\frac{\pi^2 t^2}{3\delta} \right)\right).\label{eq:beta_t_2}
\end{align} 
Then Algorithm~\ref{alg:linucb2} guarantees w.p. $> 1-\delta$ simultaneously for all $T=1,2,...$
\begin{align*}
R_T &\leq F +c^*+ \sqrt{\frac{8 (T-1) \beta_{T-1} (d+1)}{(1-\rho)^2} \log \left( 1 + \frac{T C^2_b (C^2_w+F^2) }{d \sigma^2 }\right)}.
\end{align*}
\end{theorem}

\begin{remark}
The result again shows that LinUCBw algorithm achieves $\tilde{O}(\sqrt{T})$ cumulative regret and thus it is also a no-regret algorithm under the weaker condition (Definition \ref{def:lm_weak}). Note Definition \ref{def:lm_weak} is quite weak which even doesn't require the true function sits within the approximation function class.
\end{remark}

\begin{proof}

The analysis is similar to the $\rho$-gap-adjusted case but includes $c^*=f^*-f^*_w$. For instance, let $\Delta^w_t$ denote the deviation term of our linear function from the true function at $x_t$, then
\begin{align*}
\Delta^w_t = f_0(x_t) - w^\top_* x_t-c^*,
\end{align*}
And our observation model (eq. \eqref{eq:obs}) becomes
\begin{align*}
y_t = f_0(x_t) + \eta_t = w_*^\top x_t + c^* + \Delta^w_t + \eta_t.
\end{align*}
Then similar to Lemma~\ref{lem:delta}, we have the following lemma, whose proof is nearly identical to Lemma~\ref{lem:delta}.
\begin{lemma}[Bound of deviation term]
$\forall t \in \{0,1,\ldots,T-1\}$,
\begin{align*}
|\Delta_t | \leq \frac{\rho}{1-\rho} w^\top_*(x_* - x_t).
\end{align*}
\end{lemma}

We also provide the following lemma, which is the counterpart of Lemma~\ref{lem:gap}.

\begin{lemma}
Define $u_t = \left \|\begin{bmatrix}x_t\\1\end{bmatrix} \right\|_{\Sigma_t^{-1}}$ and assume $\beta_t$ is chosen such that $w_*\in \mathrm{Ball}_t$.
Then
\begin{align*}
w_*^\top (x_* - x_t) \leq 2 \sqrt{\beta_t} u_t.
\end{align*}
\end{lemma}
\begin{proof}
Let $\tilde{w},\tilde{c}$ denote the parameter that achieves $\argmax_{w,c \in \mathrm{Ball}_t} w^\top x_t+c$, by the optimality of $x_t$, 
\begin{align*}
w_*^\top x_* - w^\top_* x_t &=\begin{bmatrix}w_*^\top,c^*\end{bmatrix} \begin{bmatrix}x_*\\1\end{bmatrix}-\begin{bmatrix}w_*^\top,c^*\end{bmatrix} \begin{bmatrix}x_t\\1\end{bmatrix}\\
&\leq \begin{bmatrix}\tilde{w}^\top,\tilde{c}\end{bmatrix} \begin{bmatrix}x_t\\1\end{bmatrix} - \begin{bmatrix}w_*^\top,c^*\end{bmatrix} \begin{bmatrix}x_t\\1\end{bmatrix}\\
&= (\begin{bmatrix}\tilde{w}^\top,\tilde{c}\end{bmatrix} - \begin{bmatrix}\hat{w}_t^\top,\hat{c}_t\end{bmatrix}+\begin{bmatrix}\hat{w}_t^\top,\hat{c}_t\end{bmatrix}-\begin{bmatrix}w_*^\top,c^*\end{bmatrix}) \begin{bmatrix}x_t\\1\end{bmatrix}\\
&\leq \left \|\begin{bmatrix}\tilde{w}^\top,\tilde{c}\end{bmatrix} - \begin{bmatrix}\hat{w}_t^\top,\hat{c}_t\end{bmatrix}\right\|_{\Sigma_t} \left \|\begin{bmatrix}x_t\\1\end{bmatrix} \right \|_{\Sigma^{-1}_t} + \left \|\begin{bmatrix}\hat{w}_t^\top,\hat{c}_t\end{bmatrix}-\begin{bmatrix}w_*^\top,c^*\end{bmatrix}\right \|_{\Sigma_t} \left \|\begin{bmatrix}x_t\\1\end{bmatrix}\right \|_{\Sigma^{-1}_t}\\
&\leq 2\sqrt{\beta_t} u_t
\end{align*}
where the second inequality applies Holder's inequality; the last line uses the definition of $\mathrm{Ball}_t$ (note that both $\begin{bmatrix}\tilde{w}^\top,\tilde{c}\end{bmatrix},\begin{bmatrix}w_*^\top,c^*\end{bmatrix}\in \mathrm{Ball}_t).$
\end{proof} 

The rest of the analysis follows the analysis of Theorem~\ref{thm:main}.
\end{proof}

\section{Unified Misspecified Bandits Framework}\label{app:unified}

For completeness, in this section, we propose a unified misspecified bandits framework, which is able to unify $\epsilon$-uniform misspecification and $\rho$-gap-adjusted misspecification. Formally, it is defined as follows.

\begin{definition}[($\rho,\epsilon$)-Gap-Adjusted Misspecification (($\rho,\epsilon$)-GAM) ]\label{def:gap_adj_unified}
Denote $f^*=\max_{x\in\mathcal{X}} f(x)$. Then we say $f$ is $(\rho,\epsilon)$-gap-adjusted misspecification approximation of $f_0$ for parameters $0 \leq \rho < 1, \epsilon > 0$ if $\forall x \in \cX$,
\begin{align*}
|f(x) - f_0(x)| \leq \rho (f^* - f_0(x)) + \epsilon.
\end{align*}
\end{definition}
\begin{remark}
When $\rho=0$, it reduces to $\epsilon$-uniform misspecification; when $\epsilon=0$, it reduces to $\rho$-GAM; and when $\rho=\epsilon=0$, it reduces to realizable setting. Note here misspecification error is mainly captured by the $\rho (f^* - f_0(x))$ term and $\epsilon$ is only the misspecification error at $x_*$, thus it is much smaller than the uniform misspecification error all over the function domain.
\end{remark}

Under Definition \ref{def:gap_adj_unified}, Algorithm \ref{alg:pe} has the following regret guarantee.

\begin{theorem}[Unified framework regret bound]\label{thm:unified} 
Suppose Assumptions \ref{ass:bound}, \ref{ass:one}, \& \ref{ass:rho2} hold and $\alpha = 1/(kT)$. Then Algorithm \ref{alg:pe} with line 6 replaced by
\begin{align*}
    \cX \leftarrow \left \{ x \in \cX: \max_{b \in \cX} \hat{w}^\top (b - x) \leq 16 \sqrt{\frac{d\log (\frac{1}{\alpha})}{m}} + 12 \sqrt{2d} \epsilon 
\right \}
\end{align*}
guarantees $\forall \ T \geq 1$,
\begin{align*}
R_T \leq O \left(\sqrt{d T \log (kT)} + \epsilon \sqrt{d} T\right).
\end{align*}
\end{theorem}
\begin{remark}

Theorem \ref{thm:unified} generalizes the result of Theorem \ref{thm:main} by introducing uniform misspecification error term. Theorem \ref{thm:unified} also generalizes Proposition 5.1 of \citet{lattimore2020learning} by introducing the gap-adjusted misspecification error term.  
\end{remark}
\begin{proof}
The proof closely follows that for Theorem \ref{thm:main}. 

First, we prove the below statement through induction:
The optimal action $x_*$ is never eliminated and the suboptimality of uneliminated arms after the batch with $m$ epsiodes is bounded by $16\zeta \sqrt{d\log \left(\frac{1}{\alpha}\right)/m} + s'\sqrt{d} \epsilon$, where $\zeta$ is the same as in the proof of Theorem \ref{thm:main}. Based on the induction assumption, we have
\begin{align*}
|x^\top(\hat{w}-w_*)| &\leq 2 \sqrt{\frac{d\log \left(\frac{1}{\alpha}\right)}{m}} + \sqrt{2d} \max_{a \in \cX} |\xi_a|\\
&\leq 2 \sqrt{\frac{d\log \left(\frac{1}{\alpha}\right)}{m}} + \sqrt{2d} \left(\rho \left(16\zeta \sqrt{\frac{d\log \left(\frac{1}{\alpha}\right)}{m/2}} + s'\sqrt{d} \epsilon \right) + \epsilon \right)\\
&\leq 8\sqrt{\frac{d\log \left(\frac{1}{\alpha}\right)}{m}} + \sqrt{2d}(\epsilon + \rho s'\sqrt{d} \epsilon)
\end{align*}

Therefore, it holds that
\begin{align*}
\max_{b \in \cX} \hat{w}^\top(b - x_*) &= \hat{w}^\top (\hat{x}-x_*)\\
&\leq w_*^\top (\hat{x} - x_*) + 16 \sqrt{\frac{d\log \left(\frac{1}{\alpha}\right)}{m}} + \sqrt{d} \epsilon \left(2\sqrt{2} + \frac{\sqrt{2}}{8}s' \right)\\
&\leq 2\epsilon + 16 \sqrt{\frac{d\log \left(\frac{1}{\alpha}\right)}{m}} + \sqrt{d} \epsilon \left(2\sqrt{2} + \frac{\sqrt{2}}{8}s' \right)\\
&\leq \sqrt{2d} \epsilon + 16\sqrt{\frac{d\log \left(\frac{1}{\alpha}\right)}{m}} + \sqrt{d} \epsilon \left(2\sqrt{2} + \frac{\sqrt{2}}{8}s' \right),
\end{align*}
which requires \textbf{Condition 4}:
\begin{align*}
3\sqrt{2} + \frac{\sqrt{2}}{8}s' \leq s.
\end{align*}

If $x$ is not eliminated after $m$ episodes,
\begin{align*}
16 \sqrt{\frac{d \log \left(\frac{1}{\alpha}\right)}{m}} + s\sqrt{d} \epsilon & \geq \hat{w}^\top (x_* - x)\\
& \geq w^\top_* (x_* - x) - 16 \sqrt{\frac{d\log \left(\frac{1}{\alpha}\right)}{m}} - s \sqrt{d} \epsilon\\
& \geq f^* - \epsilon - w^\top_* x - 16\sqrt{\frac{d\log \left(\frac{1}{\alpha}\right)}{m}} - s\sqrt{d} \epsilon\\
& \geq f^* - f_0(x) - 2\epsilon - \rho s' \sqrt{d} \epsilon - s\sqrt{d}\epsilon - 16\sqrt{\frac{d\log \left(\frac{1}{\alpha}\right)}{m}} - 16\rho \zeta \sqrt{\frac{d\log \left(\frac{1}{\alpha}\right)}{m/2}},
\end{align*}
where the last inequality is due to
\begin{align}
|f_0(x) - w^\top_* x| &\leq \rho(f^* - f_0(x)) + \epsilon\\
& \leq 16\rho \zeta \sqrt{\frac{d\log \left(\frac{1}{\alpha}\right)}{m/2}} + \rho s' \sqrt{d} \epsilon + \epsilon.
\end{align}

Therefore,
\begin{align}
f^* - f_0(x) \leq 16\zeta\sqrt{\frac{d\log \left(\frac{1}{\alpha}\right)}{m}} + \sqrt{d} \epsilon (2s + \rho s' + \sqrt{2}),
\end{align}
which requires \textbf{Condition 5}:
\begin{align}
2s + \rho s' + \sqrt{2} \leq s'.
\end{align}

Consider \textbf{Condition 4} and \textbf{Condition 5} together and it suffices to choose
\begin{align}
s = 12\sqrt{2} \quad \mathrm{and} \quad s'=48\sqrt{2}.
\end{align}

The remaining proof follows the proof in Theorem \ref{thm:main} by combining the episodes together.
\end{proof}

\end{document}